\def\year{2017}\relax
\documentclass[letterpaper]{article} 
\usepackage{aaai17}  
\usepackage{times}  
\usepackage{helvet}  
\usepackage{courier}  
\usepackage{url}  
\usepackage{graphicx}  
\usepackage{clrscode3e}
\usepackage{amsthm}
\usepackage{amssymb}
\usepackage{amsmath}
\usepackage{graphicx}
\usepackage{graphics}
\usepackage{caption}
\usepackage{standalone}
\usepackage[draft]{minted}

\usepackage{hhline}
\usepackage[rgb]{xcolor}
\usepackage{colortbl}
\definecolor{Gray}{gray}{0.85}
\newcolumntype{g}{>{\columncolor{Gray}}c}
\newcolumntype{w}{>{\columncolor{white}}c}

\nocopyright

\usepackage{todonotes}

\theoremstyle{plain}\newtheorem{thm}{Theorem}
\theoremstyle{definition}
\theoremstyle{plain}
\theoremstyle{plain}

\newcommand{\algname}{{\sc strips}tream}
\newcommand{\strips}{{\sc strips}}
\newcommand{\pddl}{{\sc pddl}}
\newcommand{\eager}{incremental}
\newcommand{\Eager}{Incremental}
\newcommand{\focused}{focused}

\makeatletter
\def\PYG@reset{\let\PYG@it=\relax \let\PYG@bf=\relax%
    \let\PYG@ul=\relax \let\PYG@tc=\relax%
    \let\PYG@bc=\relax \let\PYG@ff=\relax}
\def\PYG@tok#1{\csname PYG@tok@#1\endcsname}
\def\PYG@toks#1+{\ifx\relax#1\empty\else%
    \PYG@tok{#1}\expandafter\PYG@toks\fi}
\def\PYG@do#1{\PYG@bc{\PYG@tc{\PYG@ul{%
    \PYG@it{\PYG@bf{\PYG@ff{#1}}}}}}}
\def\PYG#1#2{\PYG@reset\PYG@toks#1+\relax+\PYG@do{#2}}

\expandafter\def\csname PYG@tok@gd\endcsname{\def\PYG@tc##1{\textcolor[rgb]{0.63,0.00,0.00}{##1}}}
\expandafter\def\csname PYG@tok@gu\endcsname{\let\PYG@bf=\textbf\def\PYG@tc##1{\textcolor[rgb]{0.50,0.00,0.50}{##1}}}
\expandafter\def\csname PYG@tok@gt\endcsname{\def\PYG@tc##1{\textcolor[rgb]{0.00,0.27,0.87}{##1}}}
\expandafter\def\csname PYG@tok@gs\endcsname{\let\PYG@bf=\textbf}
\expandafter\def\csname PYG@tok@gr\endcsname{\def\PYG@tc##1{\textcolor[rgb]{1.00,0.00,0.00}{##1}}}
\expandafter\def\csname PYG@tok@cm\endcsname{\let\PYG@it=\textit\def\PYG@tc##1{\textcolor[rgb]{0.25,0.50,0.50}{##1}}}
\expandafter\def\csname PYG@tok@vg\endcsname{\def\PYG@tc##1{\textcolor[rgb]{0.10,0.09,0.49}{##1}}}
\expandafter\def\csname PYG@tok@vi\endcsname{\def\PYG@tc##1{\textcolor[rgb]{0.10,0.09,0.49}{##1}}}
\expandafter\def\csname PYG@tok@mh\endcsname{\def\PYG@tc##1{\textcolor[rgb]{0.40,0.40,0.40}{##1}}}
\expandafter\def\csname PYG@tok@cs\endcsname{\let\PYG@it=\textit\def\PYG@tc##1{\textcolor[rgb]{0.25,0.50,0.50}{##1}}}
\expandafter\def\csname PYG@tok@ge\endcsname{\let\PYG@it=\textit}
\expandafter\def\csname PYG@tok@vc\endcsname{\def\PYG@tc##1{\textcolor[rgb]{0.10,0.09,0.49}{##1}}}
\expandafter\def\csname PYG@tok@il\endcsname{\def\PYG@tc##1{\textcolor[rgb]{0.40,0.40,0.40}{##1}}}
\expandafter\def\csname PYG@tok@go\endcsname{\def\PYG@tc##1{\textcolor[rgb]{0.53,0.53,0.53}{##1}}}
\expandafter\def\csname PYG@tok@cp\endcsname{\def\PYG@tc##1{\textcolor[rgb]{0.74,0.48,0.00}{##1}}}
\expandafter\def\csname PYG@tok@gi\endcsname{\def\PYG@tc##1{\textcolor[rgb]{0.00,0.63,0.00}{##1}}}
\expandafter\def\csname PYG@tok@gh\endcsname{\let\PYG@bf=\textbf\def\PYG@tc##1{\textcolor[rgb]{0.00,0.00,0.50}{##1}}}
\expandafter\def\csname PYG@tok@ni\endcsname{\let\PYG@bf=\textbf\def\PYG@tc##1{\textcolor[rgb]{0.60,0.60,0.60}{##1}}}
\expandafter\def\csname PYG@tok@nl\endcsname{\def\PYG@tc##1{\textcolor[rgb]{0.63,0.63,0.00}{##1}}}
\expandafter\def\csname PYG@tok@nn\endcsname{\let\PYG@bf=\textbf\def\PYG@tc##1{\textcolor[rgb]{0.00,0.00,1.00}{##1}}}
\expandafter\def\csname PYG@tok@no\endcsname{\def\PYG@tc##1{\textcolor[rgb]{0.53,0.00,0.00}{##1}}}
\expandafter\def\csname PYG@tok@na\endcsname{\def\PYG@tc##1{\textcolor[rgb]{0.49,0.56,0.16}{##1}}}
\expandafter\def\csname PYG@tok@nb\endcsname{\def\PYG@tc##1{\textcolor[rgb]{0.00,0.50,0.00}{##1}}}
\expandafter\def\csname PYG@tok@nc\endcsname{\let\PYG@bf=\textbf\def\PYG@tc##1{\textcolor[rgb]{0.00,0.00,1.00}{##1}}}
\expandafter\def\csname PYG@tok@nd\endcsname{\def\PYG@tc##1{\textcolor[rgb]{0.67,0.13,1.00}{##1}}}
\expandafter\def\csname PYG@tok@ne\endcsname{\let\PYG@bf=\textbf\def\PYG@tc##1{\textcolor[rgb]{0.82,0.25,0.23}{##1}}}
\expandafter\def\csname PYG@tok@nf\endcsname{\def\PYG@tc##1{\textcolor[rgb]{0.00,0.00,1.00}{##1}}}
\expandafter\def\csname PYG@tok@si\endcsname{\let\PYG@bf=\textbf\def\PYG@tc##1{\textcolor[rgb]{0.73,0.40,0.53}{##1}}}
\expandafter\def\csname PYG@tok@s2\endcsname{\def\PYG@tc##1{\textcolor[rgb]{0.73,0.13,0.13}{##1}}}
\expandafter\def\csname PYG@tok@nt\endcsname{\let\PYG@bf=\textbf\def\PYG@tc##1{\textcolor[rgb]{0.00,0.50,0.00}{##1}}}
\expandafter\def\csname PYG@tok@nv\endcsname{\def\PYG@tc##1{\textcolor[rgb]{0.10,0.09,0.49}{##1}}}
\expandafter\def\csname PYG@tok@s1\endcsname{\def\PYG@tc##1{\textcolor[rgb]{0.73,0.13,0.13}{##1}}}
\expandafter\def\csname PYG@tok@ch\endcsname{\let\PYG@it=\textit\def\PYG@tc##1{\textcolor[rgb]{0.25,0.50,0.50}{##1}}}
\expandafter\def\csname PYG@tok@m\endcsname{\def\PYG@tc##1{\textcolor[rgb]{0.40,0.40,0.40}{##1}}}
\expandafter\def\csname PYG@tok@gp\endcsname{\let\PYG@bf=\textbf\def\PYG@tc##1{\textcolor[rgb]{0.00,0.00,0.50}{##1}}}
\expandafter\def\csname PYG@tok@sh\endcsname{\def\PYG@tc##1{\textcolor[rgb]{0.73,0.13,0.13}{##1}}}
\expandafter\def\csname PYG@tok@ow\endcsname{\let\PYG@bf=\textbf\def\PYG@tc##1{\textcolor[rgb]{0.67,0.13,1.00}{##1}}}
\expandafter\def\csname PYG@tok@sx\endcsname{\def\PYG@tc##1{\textcolor[rgb]{0.00,0.50,0.00}{##1}}}
\expandafter\def\csname PYG@tok@bp\endcsname{\def\PYG@tc##1{\textcolor[rgb]{0.00,0.50,0.00}{##1}}}
\expandafter\def\csname PYG@tok@c1\endcsname{\let\PYG@it=\textit\def\PYG@tc##1{\textcolor[rgb]{0.25,0.50,0.50}{##1}}}
\expandafter\def\csname PYG@tok@o\endcsname{\def\PYG@tc##1{\textcolor[rgb]{0.40,0.40,0.40}{##1}}}
\expandafter\def\csname PYG@tok@kc\endcsname{\let\PYG@bf=\textbf\def\PYG@tc##1{\textcolor[rgb]{0.00,0.50,0.00}{##1}}}
\expandafter\def\csname PYG@tok@c\endcsname{\let\PYG@it=\textit\def\PYG@tc##1{\textcolor[rgb]{0.25,0.50,0.50}{##1}}}
\expandafter\def\csname PYG@tok@mf\endcsname{\def\PYG@tc##1{\textcolor[rgb]{0.40,0.40,0.40}{##1}}}
\expandafter\def\csname PYG@tok@err\endcsname{\def\PYG@bc##1{\setlength{\fboxsep}{0pt}\fcolorbox[rgb]{1.00,0.00,0.00}{1,1,1}{\strut ##1}}}
\expandafter\def\csname PYG@tok@mb\endcsname{\def\PYG@tc##1{\textcolor[rgb]{0.40,0.40,0.40}{##1}}}
\expandafter\def\csname PYG@tok@ss\endcsname{\def\PYG@tc##1{\textcolor[rgb]{0.10,0.09,0.49}{##1}}}
\expandafter\def\csname PYG@tok@sr\endcsname{\def\PYG@tc##1{\textcolor[rgb]{0.73,0.40,0.53}{##1}}}
\expandafter\def\csname PYG@tok@mo\endcsname{\def\PYG@tc##1{\textcolor[rgb]{0.40,0.40,0.40}{##1}}}
\expandafter\def\csname PYG@tok@kd\endcsname{\let\PYG@bf=\textbf\def\PYG@tc##1{\textcolor[rgb]{0.00,0.50,0.00}{##1}}}
\expandafter\def\csname PYG@tok@mi\endcsname{\def\PYG@tc##1{\textcolor[rgb]{0.40,0.40,0.40}{##1}}}
\expandafter\def\csname PYG@tok@kn\endcsname{\let\PYG@bf=\textbf\def\PYG@tc##1{\textcolor[rgb]{0.00,0.50,0.00}{##1}}}
\expandafter\def\csname PYG@tok@cpf\endcsname{\let\PYG@it=\textit\def\PYG@tc##1{\textcolor[rgb]{0.25,0.50,0.50}{##1}}}
\expandafter\def\csname PYG@tok@kr\endcsname{\let\PYG@bf=\textbf\def\PYG@tc##1{\textcolor[rgb]{0.00,0.50,0.00}{##1}}}
\expandafter\def\csname PYG@tok@s\endcsname{\def\PYG@tc##1{\textcolor[rgb]{0.73,0.13,0.13}{##1}}}
\expandafter\def\csname PYG@tok@kp\endcsname{\def\PYG@tc##1{\textcolor[rgb]{0.00,0.50,0.00}{##1}}}
\expandafter\def\csname PYG@tok@w\endcsname{\def\PYG@tc##1{\textcolor[rgb]{0.73,0.73,0.73}{##1}}}
\expandafter\def\csname PYG@tok@kt\endcsname{\def\PYG@tc##1{\textcolor[rgb]{0.69,0.00,0.25}{##1}}}
\expandafter\def\csname PYG@tok@sc\endcsname{\def\PYG@tc##1{\textcolor[rgb]{0.73,0.13,0.13}{##1}}}
\expandafter\def\csname PYG@tok@sb\endcsname{\def\PYG@tc##1{\textcolor[rgb]{0.73,0.13,0.13}{##1}}}
\expandafter\def\csname PYG@tok@k\endcsname{\let\PYG@bf=\textbf\def\PYG@tc##1{\textcolor[rgb]{0.00,0.50,0.00}{##1}}}
\expandafter\def\csname PYG@tok@se\endcsname{\let\PYG@bf=\textbf\def\PYG@tc##1{\textcolor[rgb]{0.73,0.40,0.13}{##1}}}
\expandafter\def\csname PYG@tok@sd\endcsname{\let\PYG@it=\textit\def\PYG@tc##1{\textcolor[rgb]{0.73,0.13,0.13}{##1}}}

\def\PYGZbs{\char`\\}
\def\PYGZus{\char`\_}
\def\PYGZob{\char`\{}
\def\PYGZcb{\char`\}}

\def\PYGZsh{\char`\#}
\def\PYGZpc{\char`\%}

\def\PYGZsq{\char`\'}

\makeatother

\makeatletter
\def\PYGdefault@reset{\let\PYGdefault@it=\relax \let\PYGdefault@bf=\relax%
    \let\PYGdefault@ul=\relax \let\PYGdefault@tc=\relax%
    \let\PYGdefault@bc=\relax \let\PYGdefault@ff=\relax}
\def\PYGdefault@tok#1{\csname PYGdefault@tok@#1\endcsname}
\def\PYGdefault@toks#1+{\ifx\relax#1\empty\else%
    \PYGdefault@tok{#1}\expandafter\PYGdefault@toks\fi}
\def\PYGdefault@do#1{\PYGdefault@bc{\PYGdefault@tc{\PYGdefault@ul{%
    \PYGdefault@it{\PYGdefault@bf{\PYGdefault@ff{#1}}}}}}}
\def\PYGdefault#1#2{\PYGdefault@reset\PYGdefault@toks#1+\relax+\PYGdefault@do{#2}}

\expandafter\def\csname PYGdefault@tok@gd\endcsname{\def\PYGdefault@tc##1{\textcolor[rgb]{0.63,0.00,0.00}{##1}}}
\expandafter\def\csname PYGdefault@tok@gu\endcsname{\let\PYGdefault@bf=\textbf\def\PYGdefault@tc##1{\textcolor[rgb]{0.50,0.00,0.50}{##1}}}
\expandafter\def\csname PYGdefault@tok@gt\endcsname{\def\PYGdefault@tc##1{\textcolor[rgb]{0.00,0.27,0.87}{##1}}}
\expandafter\def\csname PYGdefault@tok@gs\endcsname{\let\PYGdefault@bf=\textbf}
\expandafter\def\csname PYGdefault@tok@gr\endcsname{\def\PYGdefault@tc##1{\textcolor[rgb]{1.00,0.00,0.00}{##1}}}
\expandafter\def\csname PYGdefault@tok@cm\endcsname{\let\PYGdefault@it=\textit\def\PYGdefault@tc##1{\textcolor[rgb]{0.25,0.50,0.50}{##1}}}
\expandafter\def\csname PYGdefault@tok@vg\endcsname{\def\PYGdefault@tc##1{\textcolor[rgb]{0.10,0.09,0.49}{##1}}}
\expandafter\def\csname PYGdefault@tok@vi\endcsname{\def\PYGdefault@tc##1{\textcolor[rgb]{0.10,0.09,0.49}{##1}}}
\expandafter\def\csname PYGdefault@tok@mh\endcsname{\def\PYGdefault@tc##1{\textcolor[rgb]{0.40,0.40,0.40}{##1}}}
\expandafter\def\csname PYGdefault@tok@cs\endcsname{\let\PYGdefault@it=\textit\def\PYGdefault@tc##1{\textcolor[rgb]{0.25,0.50,0.50}{##1}}}
\expandafter\def\csname PYGdefault@tok@ge\endcsname{\let\PYGdefault@it=\textit}
\expandafter\def\csname PYGdefault@tok@vc\endcsname{\def\PYGdefault@tc##1{\textcolor[rgb]{0.10,0.09,0.49}{##1}}}
\expandafter\def\csname PYGdefault@tok@il\endcsname{\def\PYGdefault@tc##1{\textcolor[rgb]{0.40,0.40,0.40}{##1}}}
\expandafter\def\csname PYGdefault@tok@go\endcsname{\def\PYGdefault@tc##1{\textcolor[rgb]{0.53,0.53,0.53}{##1}}}
\expandafter\def\csname PYGdefault@tok@cp\endcsname{\def\PYGdefault@tc##1{\textcolor[rgb]{0.74,0.48,0.00}{##1}}}
\expandafter\def\csname PYGdefault@tok@gi\endcsname{\def\PYGdefault@tc##1{\textcolor[rgb]{0.00,0.63,0.00}{##1}}}
\expandafter\def\csname PYGdefault@tok@gh\endcsname{\let\PYGdefault@bf=\textbf\def\PYGdefault@tc##1{\textcolor[rgb]{0.00,0.00,0.50}{##1}}}
\expandafter\def\csname PYGdefault@tok@ni\endcsname{\let\PYGdefault@bf=\textbf\def\PYGdefault@tc##1{\textcolor[rgb]{0.60,0.60,0.60}{##1}}}
\expandafter\def\csname PYGdefault@tok@nl\endcsname{\def\PYGdefault@tc##1{\textcolor[rgb]{0.63,0.63,0.00}{##1}}}
\expandafter\def\csname PYGdefault@tok@nn\endcsname{\let\PYGdefault@bf=\textbf\def\PYGdefault@tc##1{\textcolor[rgb]{0.00,0.00,1.00}{##1}}}
\expandafter\def\csname PYGdefault@tok@no\endcsname{\def\PYGdefault@tc##1{\textcolor[rgb]{0.53,0.00,0.00}{##1}}}
\expandafter\def\csname PYGdefault@tok@na\endcsname{\def\PYGdefault@tc##1{\textcolor[rgb]{0.49,0.56,0.16}{##1}}}
\expandafter\def\csname PYGdefault@tok@nb\endcsname{\def\PYGdefault@tc##1{\textcolor[rgb]{0.00,0.50,0.00}{##1}}}
\expandafter\def\csname PYGdefault@tok@nc\endcsname{\let\PYGdefault@bf=\textbf\def\PYGdefault@tc##1{\textcolor[rgb]{0.00,0.00,1.00}{##1}}}
\expandafter\def\csname PYGdefault@tok@nd\endcsname{\def\PYGdefault@tc##1{\textcolor[rgb]{0.67,0.13,1.00}{##1}}}
\expandafter\def\csname PYGdefault@tok@ne\endcsname{\let\PYGdefault@bf=\textbf\def\PYGdefault@tc##1{\textcolor[rgb]{0.82,0.25,0.23}{##1}}}
\expandafter\def\csname PYGdefault@tok@nf\endcsname{\def\PYGdefault@tc##1{\textcolor[rgb]{0.00,0.00,1.00}{##1}}}
\expandafter\def\csname PYGdefault@tok@si\endcsname{\let\PYGdefault@bf=\textbf\def\PYGdefault@tc##1{\textcolor[rgb]{0.73,0.40,0.53}{##1}}}
\expandafter\def\csname PYGdefault@tok@s2\endcsname{\def\PYGdefault@tc##1{\textcolor[rgb]{0.73,0.13,0.13}{##1}}}
\expandafter\def\csname PYGdefault@tok@nt\endcsname{\let\PYGdefault@bf=\textbf\def\PYGdefault@tc##1{\textcolor[rgb]{0.00,0.50,0.00}{##1}}}
\expandafter\def\csname PYGdefault@tok@nv\endcsname{\def\PYGdefault@tc##1{\textcolor[rgb]{0.10,0.09,0.49}{##1}}}
\expandafter\def\csname PYGdefault@tok@s1\endcsname{\def\PYGdefault@tc##1{\textcolor[rgb]{0.73,0.13,0.13}{##1}}}
\expandafter\def\csname PYGdefault@tok@ch\endcsname{\let\PYGdefault@it=\textit\def\PYGdefault@tc##1{\textcolor[rgb]{0.25,0.50,0.50}{##1}}}
\expandafter\def\csname PYGdefault@tok@m\endcsname{\def\PYGdefault@tc##1{\textcolor[rgb]{0.40,0.40,0.40}{##1}}}
\expandafter\def\csname PYGdefault@tok@gp\endcsname{\let\PYGdefault@bf=\textbf\def\PYGdefault@tc##1{\textcolor[rgb]{0.00,0.00,0.50}{##1}}}
\expandafter\def\csname PYGdefault@tok@sh\endcsname{\def\PYGdefault@tc##1{\textcolor[rgb]{0.73,0.13,0.13}{##1}}}
\expandafter\def\csname PYGdefault@tok@ow\endcsname{\let\PYGdefault@bf=\textbf\def\PYGdefault@tc##1{\textcolor[rgb]{0.67,0.13,1.00}{##1}}}
\expandafter\def\csname PYGdefault@tok@sx\endcsname{\def\PYGdefault@tc##1{\textcolor[rgb]{0.00,0.50,0.00}{##1}}}
\expandafter\def\csname PYGdefault@tok@bp\endcsname{\def\PYGdefault@tc##1{\textcolor[rgb]{0.00,0.50,0.00}{##1}}}
\expandafter\def\csname PYGdefault@tok@c1\endcsname{\let\PYGdefault@it=\textit\def\PYGdefault@tc##1{\textcolor[rgb]{0.25,0.50,0.50}{##1}}}
\expandafter\def\csname PYGdefault@tok@o\endcsname{\def\PYGdefault@tc##1{\textcolor[rgb]{0.40,0.40,0.40}{##1}}}
\expandafter\def\csname PYGdefault@tok@kc\endcsname{\let\PYGdefault@bf=\textbf\def\PYGdefault@tc##1{\textcolor[rgb]{0.00,0.50,0.00}{##1}}}
\expandafter\def\csname PYGdefault@tok@c\endcsname{\let\PYGdefault@it=\textit\def\PYGdefault@tc##1{\textcolor[rgb]{0.25,0.50,0.50}{##1}}}
\expandafter\def\csname PYGdefault@tok@mf\endcsname{\def\PYGdefault@tc##1{\textcolor[rgb]{0.40,0.40,0.40}{##1}}}
\expandafter\def\csname PYGdefault@tok@err\endcsname{\def\PYGdefault@bc##1{\setlength{\fboxsep}{0pt}\fcolorbox[rgb]{1.00,0.00,0.00}{1,1,1}{\strut ##1}}}
\expandafter\def\csname PYGdefault@tok@mb\endcsname{\def\PYGdefault@tc##1{\textcolor[rgb]{0.40,0.40,0.40}{##1}}}
\expandafter\def\csname PYGdefault@tok@ss\endcsname{\def\PYGdefault@tc##1{\textcolor[rgb]{0.10,0.09,0.49}{##1}}}
\expandafter\def\csname PYGdefault@tok@sr\endcsname{\def\PYGdefault@tc##1{\textcolor[rgb]{0.73,0.40,0.53}{##1}}}
\expandafter\def\csname PYGdefault@tok@mo\endcsname{\def\PYGdefault@tc##1{\textcolor[rgb]{0.40,0.40,0.40}{##1}}}
\expandafter\def\csname PYGdefault@tok@kd\endcsname{\let\PYGdefault@bf=\textbf\def\PYGdefault@tc##1{\textcolor[rgb]{0.00,0.50,0.00}{##1}}}
\expandafter\def\csname PYGdefault@tok@mi\endcsname{\def\PYGdefault@tc##1{\textcolor[rgb]{0.40,0.40,0.40}{##1}}}
\expandafter\def\csname PYGdefault@tok@kn\endcsname{\let\PYGdefault@bf=\textbf\def\PYGdefault@tc##1{\textcolor[rgb]{0.00,0.50,0.00}{##1}}}
\expandafter\def\csname PYGdefault@tok@cpf\endcsname{\let\PYGdefault@it=\textit\def\PYGdefault@tc##1{\textcolor[rgb]{0.25,0.50,0.50}{##1}}}
\expandafter\def\csname PYGdefault@tok@kr\endcsname{\let\PYGdefault@bf=\textbf\def\PYGdefault@tc##1{\textcolor[rgb]{0.00,0.50,0.00}{##1}}}
\expandafter\def\csname PYGdefault@tok@s\endcsname{\def\PYGdefault@tc##1{\textcolor[rgb]{0.73,0.13,0.13}{##1}}}
\expandafter\def\csname PYGdefault@tok@kp\endcsname{\def\PYGdefault@tc##1{\textcolor[rgb]{0.00,0.50,0.00}{##1}}}
\expandafter\def\csname PYGdefault@tok@w\endcsname{\def\PYGdefault@tc##1{\textcolor[rgb]{0.73,0.73,0.73}{##1}}}
\expandafter\def\csname PYGdefault@tok@kt\endcsname{\def\PYGdefault@tc##1{\textcolor[rgb]{0.69,0.00,0.25}{##1}}}
\expandafter\def\csname PYGdefault@tok@sc\endcsname{\def\PYGdefault@tc##1{\textcolor[rgb]{0.73,0.13,0.13}{##1}}}
\expandafter\def\csname PYGdefault@tok@sb\endcsname{\def\PYGdefault@tc##1{\textcolor[rgb]{0.73,0.13,0.13}{##1}}}
\expandafter\def\csname PYGdefault@tok@k\endcsname{\let\PYGdefault@bf=\textbf\def\PYGdefault@tc##1{\textcolor[rgb]{0.00,0.50,0.00}{##1}}}
\expandafter\def\csname PYGdefault@tok@se\endcsname{\let\PYGdefault@bf=\textbf\def\PYGdefault@tc##1{\textcolor[rgb]{0.73,0.40,0.13}{##1}}}
\expandafter\def\csname PYGdefault@tok@sd\endcsname{\let\PYGdefault@it=\textit\def\PYGdefault@tc##1{\textcolor[rgb]{0.73,0.13,0.13}{##1}}}

\makeatother

\begin{document}
%
\title{STRIPS Planning in Infinite Domains}

 \pdfinfo{
 /Title (STRIPS Planning in Infinite Domains)
 /Author (Caelan Reed Garrett, Tom\'as Lozano-P\'erez, and Leslie Pack
 Kaelbling)}
 
 \author{Caelan Reed Garrett, Tom\'as Lozano-P\'erez, and Leslie Pack Kaelbling \\
 MIT CSAIL\\
 32 Vassar Street \\
 Cambridge, MA 02139, USA\\
 {\tt\small \{caelan,tlp,lpk\}@csail.mit.edu}}

\maketitle


\begin{abstract}
\begin{quote}  
  Many robotic planning applications involve continuous
  actions with highly non-linear constraints, which
  cannot be modeled using modern planners that construct a
  propositional representation.  We introduce \algname{}: an
  extension of the \strips{} language which can model these domains
  by supporting the specification of blackbox generators to handle complex constraints. 
  The outputs of these generators interact with actions through possibly infinite streams of
  objects and static predicates. 
  We provide two algorithms which both
  reduce \algname{} problems to a sequence of finite-domain planning
  problems. The representation and algorithms are entirely
  domain independent.  We demonstrate our framework on simple illustrative
  domains, and then on a high-dimensional, continuous robotic task
  and motion planning domain.  
\end{quote}
\end{abstract}


\section{Introduction}

Many important planning domains naturally occur in continuous spaces involving
complex constraints among variables. Consider planning for a robot tasked with
organizing several blocks in a room. 
The robot must find a sequence of {\em pick}, {\em place}, and {\em move} actions
involving continuous robot configurations, robot trajectories, block poses, and block grasps.
These variables must satisfy highly non-linear kinematic, collision, and motion constraints 
which affect the feasibility of the actions. 
Each constraint typically requires a special purpose procedure to efficiently evaluate it or produce satisfying values for it such as an inverse kinematic solver, collision checker, or motion planner. 


We propose an approach, called \algname{}, which can model such a
domain by providing a generic interface for blackbox procedures to be incorporated
in an action language. The implementation of the procedures is abstracted away using
{\em streams}: finite or infinite sequences of objects such as poses, configurations, and trajectories.
We introduce the following 
two additional stream capabilities to effectively model
domains with complex predicates that are only true for small sets 
of their argument values:
\begin{itemize}
\item {\bf conditional streams}:  a stream of objects may be defined
  as a function of other objects; for example, a stream of possible
  positions of one object given the position of another object that it
  must be on top of or a stream of possible settings of parameters of
  a factory machine given desired properties of its output.
\item {\bf certified streams}: streams of objects may be declared not
  only to be of a specific type, but also to satisfy an arbitrary
  conjunction of predicates;  for example, one might define a
  certified conditional stream that generates positions for an object
  that satisfy requirements that the object be on a surface, that a
  robot be able to reach the object at that position, and that the
  robot be able to see the object while reaching.   
\end{itemize}

\begin{figure}[ht]
\centering
\includegraphics[width=0.4\textwidth]{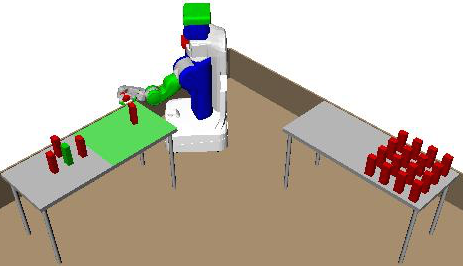}
\caption{Problem 2-16.} 
\label{fig:tamp_domains}
\end{figure}

Through streams, \algname{} can compactly model
a large class of continuous, countably infinite, and large finite domains.
By conditioning on partial argument values
and using sampling, it can even effectively
model domains where the set of valid action argument values is lower
dimensional than the possible argument space. For example, in our robotics domain, the set of 
inverse kinematics solutions for a particular pose and grasp is much lower dimensional than the full set of robot configurations. However, using a conditional
stream, we can specify an inverse kinematics solver which directly samples from this
set given a pose and grasp.


The approach is entirely domain-independent, and reduces to \strips{} in
the case of finite domains.  The only additional requirement is the
specification of a set of streams that can generate objects satisfying
the static predicates in the domain. 
It is accompanied by two algorithms, a simple
and a focused version, which operate by constructing and solving a
sequence of \strips{} planning problems.  
This strategy takes advantage
of the highly optimized search strategies and heuristics that exist
for \strips{} planning, while expanding the domain of
applicability of those techniques.
Additionally, the focused version can efficiently solve problems where using streams is computationally expensive by carefully choosing to only call potentially useful streams. 

\section{Related work}
There are a number of existing general-purpose approaches to solving planning problems in infinite domains, each of which has some significant limitation when modeling our robot domain.


Temporal planning, such as defined in {\sc pddl2.1}~\cite{Fox03pddl2.1:an}, is often formulated in terms of linear constraints on plan variables and is typically solved using techniques based on linear programming~\cite{hoffmann2003metric,coles2013hybrid}. 
{\sc pddl+}~\cite{fox2006modelling} extends {\sc pddl2.1} by introducing exogenous events and continuous processes. 
Although {\sc pddl+} supports continuous variables, the values of continuous variables are functions of the sequence of discrete actions performed at particular times. 
Thus, time is the only truly non-dependent continuous variable. 
In contrast, our motivating robot domain has no notion of time but instead a continuously infinite branching factor.
Many planners solve {\sc pddl+} problems with non-linear process models by discretizing time~\cite{della2009upmurphi,piotrowski2016heuristic}. Some recent planners can solve {\sc pddl+} problems with polynomial process models exactly without time discretization~\cite{bryce2015smt,cashmore2016compilation}.
However, even in simplified robotics domains that {\sc pddl+} can model, modern {\sc pddl+} planners are ineffective at planning with collision and kinematic constraints (both highly non-polynomial constraints), particularly in high-dimensional systems.


General-purpose lifted and first-order approaches, such as
those based on first-order situation calculus or Prolog, provide
semi-decision procedures for a large class of lifted planning
problems.  However, the generality tends to come at a huge price in
efficiency and these planning strategies are rarely practical. 




The Answer Set Programming (ASP) literature contains analysis on reasoning in infinite domains through finitary, $\omega$-restricted, finitely ground, and finite domain ASPs~\cite{bonatti2010answer}. 
The DLV-Complex system~\cite{calimeri2009asp} is able to solve feasible finitely ground programs by extending the DataLog with Disjunction (DLV) system to support functions, lists, and set terms.
We believe 
that the language of ASP allows specification of
conditional and certified streams.  
However, the ground ASP solver
still has to address a much more general and difficult problem and
will not have the appropriate heuristic strategies that make current
domain-independent \strips{} planners so effective. 

Semantic attachments~\cite{dornhege09icaps}, predicates computed
by an external program, also provide a way of integrating blackbox
procedures and {\sc pddl} planners. 
Because semantic attachments take 
a state as input, they can only be used in forward state-space search.
Furthermore, semantic attachments are ignored in heuristics. This results in poor planner performance, particularly when the attachments are expensive to evaluate such as in robotics domains. 
Finally, because semantic attachments are restricted to be functions, they are unable to model domains
with infinitely many possible successor states.

Many approaches to robotics planning problems, including motion
planning and task-and-motion planning, have developed strategies for
handling continuous spaces that go beyond {\em a priori}
discretization. 
Several approaches, for example~\cite{HPN,Erdem,Srivastava14,GarrettIROS15,dantam2016tmp,garrett2016ffrob}, have been suggested for
integrating these sampling-based robot motion planning methods with
symbolic planning methods.  Of these approaches, those able to plan in
realistic robot domains have typically been quite special purpose; the
more general purpose approaches have typically been less capable.  




\section{Representation}
In this section we describe the representational components of a
planning domain and problem, which include static and fluent
predicates, operators, and streams.  {\em Objects} 
serve as arguments to predicates and as
parameters to operators;  they are generated by streams.

A {\em static predicate} is a predicate which, for any tuple of
objects, has a constant truth value throughout a problem instance.
Static predicates generally serve to represent constraints on the
parameters of an operator. We restrict static predicates to only ever 
be mentioned positively because, in the general infinite case, it is not possible to 
verify that a predicate does not hold.

An {\em operator} schema is specified by a tuple of formal parameters $(X_1,
\ldots, X_n)$ and conjunctions of static positive
preconditions \kw{stat}, fluent literal preconditions \kw{pre}, and
fluent literal effects \kw{eff} and has the same semantics as in \strips.
An operator instance is a ground instantiation of an operator schema
with objects substituted in for the formal parameters.
When necessary, we augment the set of operator schemas with a set of axioms
that naively use the same schema form as operators. We assume the set
of axioms can be compiled into a set of derived predicates as used in \pddl.   

A {\em generator} $g = \langle \bar{o}^1, \bar{o}^2, ... \rangle$ is a
finite or infinite sequence of object tuples $\bar{o} = (o_1, ...,
o_n)$.
The procedure $\proc{next}(g)$ returns the next element in generator
$g$ and returns the special object $\kw{None}$ to indicate that the
stream has been exhausted and contains no more objects.
A {\em conditional generator} $f(\bar{x})$ is a function from $\bar{x}
= x_1, ..., x_n$ to a generator $g_{\bar{x}}$ which generates tuples
from a domain not necessarily the same as the domain of $\bar{x}$. 

An {\em stream} schema, $\sigma(\bar{Y} \mid \bar{X})$, is specified
by a 
tuple of input parameters $\bar{X} = (X_1, ..., X_m)$, a tuple of
output parameters $\bar{Y} = (Y_1, ..., Y_n)$, a conditional generator
$\kw{gen} =  f(\bar{X})$ defined on $\bar{X}$, a conjunction of input static
atoms $\kw{inp}$ defined on $\bar{X}$, and a conjunction of output static
atoms $\kw{out}$ defined on $\bar{X}$ and $\bar{Y}$. 
The conditional generator $f$ is a function, implemented in the host
programming language, that returns a generator such that,
for all $\bar{x}$ satisfying the conditions \kw{inp}, $\forall \bar{y}
\in f(\bar{x}), (\bar{x}, \bar{y})$ satisfy the conditions \kw{out}. 
A stream instance is a ground instantiation of a stream schema with
objects substituted in for input parameters $(X_1, \ldots, X_n)$;  it
is {\em conditioned} on those object values and, if the \kw{inp}
conditions are satisfied, then it will generate a stream of tuples of
objects each of which satisfies the certification conditions \kw{out}.



The notion of a conditional stream is quite general;  there are two
specific cases that are worth understanding in detail.
An {\em unconditional stream} $\sigma(\bar{Y} \mid ())$ is a stream
with no inputs whose associated function $f$ returns a single
generator, which might be used to generate objects of a given type,
for example, independent of whatever other objects are specified in a
domain.
A {\em test stream} $\sigma(() \mid \bar{X})$ is a degenerate, but
still useful, type of stream with no outputs.  In this case, 
$f(X_1, ..., X_m)$ contains either the single element $()$, indicating
that the \kw{inp} conditions hold of $\bar{X}$, or
contains no elements at all, indicating that the \kw{inp} conditions
do not hold of $\bar{X}$.  It can be interpreted as an implicit
Boolean test.


A {\em planning domain}
${\cal D} = ({\cal P}_s, {\cal P}_f, {\cal C}_0, {\cal A}, {\cal X},
\Sigma)$ is specified by finite sets of static predicates
${\cal P}_s$, fluent predicates ${\cal P}_f$, initial constant objects
${\cal C}_0$, operator schemas ${\cal A}$, axiom schemas ${\cal X}$, and
stream schemas $\Sigma$.  Note that the initial objects (as well as
objects generated by the streams) may in general not be simple
symbols, but can be numeric values or even structures such as matrices
or objects in an underlying programming language.  They must provide a
unique ID, such as a hash value, for use in the STRIPS planning phase.
 
A {\em \algname{} problem} $\Pi = ({\cal D}, O_0, s_0, s_*)$ is
specified by a planning domain ${\cal D}$, a finite set of initial objects
$O_0$, an initial state composed of a finite set of static or fluent
atoms $s_0$, and a goal set defined to be the set of states satisfying
fluent literals $s_*$.  
We make a version of the closed world assumption on the initial state
$s_0$, assuming that all true fluents 
are contained in it.  This
initial state will not be complete: in general, it will be impossible
to assert all true static atoms when the universe is infinite.

Let ${\cal O}_\Pi$ and ${\cal S}_\Pi$ be the universe of all objects and the set of true initial atoms
that can be generated from
a finite set $\Sigma$ of stream schemas, a finite set $C_0 \cup O_0$ of
initial objects, and initial state $s_0$. We give all proofs in the the appendix.
\begin{thm}  \label{thm:re}
${\cal O}_\Pi$ and ${\cal S}_\Pi$ are recursively enumerable (RE). 
\end{thm}
A {\em solution} to a \algname{} problem $\Pi$ is a finite sequence of operator instances $\pi_*$ with object parameters contained within ${\cal O}_\Pi$ that is applicable from ${\cal S}_\Pi$ and results in a state that satisfies $s_*$.
\algname{} is undecidable but semi-decidable, so we restrict our attention
to feasible instances.

\begin{thm}\label{thm:decide}
The existence of a solution for a \algname{} problem $\Pi$ is
undecidable. 
\end{thm}

\begin{thm}\label{thm:decide}
The existence of a solution for a \algname{} problem $\Pi$ is
semi-decidable. 
\end{thm}


\section{Planning algorithms}
We present two algorithms for solving \algname{} problems:  the {\em
  \eager{}} planner takes advantage of certified conditional streams in
the problem specification to generate the necessary objects for solving
the problem;  the {\em focused} planner adds the ability to focus the
object-generation process based on the requirements of the plan being
constructed.  Both algorithms are sound and complete: if a solution
exists they will find it in finite time.

Both planners operate iteratively, alternating between adding elements
and atoms to a current set of objects and initial atoms and constructing and
solving \strips{} planning problem instances. 
A \strips{} problem $({\cal P}, {\cal
  A}, O, s_\id{init}, s_*)$ is specified by a set of predicates, a set of
operator schemas, a set of constant symbols, an initial set of
atoms, and a set of goal literals.  
Let $\proc{s-plan}({\cal P}, {\cal A}, O, s_\id{init}, s_*)$ be any sound and complete
planner for the \strips{} subset of \pddl{}.
We implement \proc{s-plan} using FastDownward~\cite{helmert2006fast}.

\subsection{\Eager{} planner} 

The \eager{} planner maintains a queue of
stream instances $Q$ and incrementally constructs set ${\cal O}$ of
objects and set ${\cal S}$ of fluents and static atoms that are true in the
initial state.  
The {\em done} set $D$ contains all streams that have been constructed and
exhausted.  In each iteration of the main loop, a \strips{} planning
instance is constructed from the current sets ${\cal
  O}$ and ${\cal S}$, with the same predicates, operator and axiom schemas, and goal.  If a
plan is obtained, it is returned.  If not, then $K \geq 1$ attempts to add
new objects are made where $K$ is a meta-parameter.
In each one, a stream $\sigma(\bar{Y}\mid\bar{x})$
is popped from $Q$ and a new tuple of objects $\bar{y}$ is extracted
from it.  If the stream is exhausted, it is stored in $D$.  
Otherwise, the objects in $\bar{y}$ are added to
${\cal O}$, the output fluents from $\sigma$ applied to $(\bar{x},
\bar{y})$ are added to ${\cal S}$, and a new set of streams $\Sigma_n$
is constructed.  For all stream schemas $\sigma$ and possible tuples
of the appropriate size $\bar{x}'$, if the input conditions
$\sigma'.\kw{inp}(\bar{x}')$ are in ${\cal S}$, then the instantiated
stream $\sigma'(\bar{Y}'\mid{\bar x}')$ is added to $Q$ if it has not
been added previously.  We also return the stream
$\sigma(\bar{Y}\mid\bar{x})$ to $Q$ so we may revisit it in the future.
The pseudo-code is shown below.

\begin{footnotesize}
\begin{codebox}
\Procname{$\proc{\eager{}}((({\cal P}_s, {\cal P}_f,  {\cal
    C}_0, {\cal A}, {\cal X}, \Sigma), O_0, s_0, s_*),
  \proc{s-plan}, K):$} 
\zi ${\cal O} = {\cal C}_0 \cup O_0;\; {\cal S} = s_0;\; D = \emptyset$
\zi $Q = \proc{Queue}(\{\sigma(\bar{Y} \mid \bar{x}) \mid \sigma(\bar{Y} \mid \bar{X}) \in \Sigma$,
\zi \;\;\;\;\;\;\;\;\;\;\;\;\;\;\;\;\;\;\;\;\;\;$\bar{x} \in {\cal O}^{|\bar{X}|}, \sigma.\kw{inp}(\bar{x}) \subseteq {\cal S}\})$
\zi \While \kw{True}: \Do
\zi $\pi_* = \proc{s-plan}({\cal P}_s \cup {\cal P}_f, {\cal A}\cup{\cal
  X}, {\cal  O}, {\cal S}, s_*)$ 
\zi \If $\pi_* \neq \kw{None}$: \Then
\zi \Return $\pi_*$
\End
\zi \If $\proc{empty}(Q)$: \Then
\zi \Return $\kw{None}$
\End
\zi \For $k \in \{1, ..., \proc{min}(K, \proc{len}(Q))\}$: \Do
\zi $\sigma(\bar{Y} \mid \bar{x}) = \proc{pop}(Q)$
\zi $\bar{y} = \proc{next}(\sigma.f(\bar{x}))$
\zi \If $\bar{y} = \kw{None}$: \Then 
\zi $D = D \cup \{\sigma(\bar{Y} \mid \bar{x})\}$
\zi \kw{continue}
\End
\zi ${\cal O} = {\cal O} \cup \bar{y};\; {\cal S} = {\cal S} \cup
\sigma.\kw{out}((\bar{x}, \bar{y}))$ 
\zi \For $\sigma'(\bar{Y}' \mid \bar{X}') \in \Sigma$: \Do
\zi \For $\bar{x}' \in {\cal O}^{|\bar{X}'|}$: \Do
\zi \If $\sigma'.\kw{inp}(\bar{x}') \subseteq {\cal S}$, $\sigma'(\bar{Y}' \mid \bar{x}') \notin (Q \cup D)$: \Then
\zi $\proc{push}(Q, \sigma'(\bar{Y}' \mid \bar{x}'))$
\End\End\End
\zi $\proc{push}(Q, \sigma(\bar{Y} \mid \bar{x}))$
\End\End
\end{codebox}
\end{footnotesize}

In practice, many \proc{s-plan} calls report infeasibility immediately because they
have infinite admissible heuristic values. 
We prove the \eager{} algorithm is complete in the appendix.

\begin{thm}
The \eager{} algorithm is complete.
\end{thm}

\subsection{Focused planner}  

The focused planner is particularly aimed
at domains for which it is expensive to draw an object from a stream;
this occurs when the stream elements are certified to satisfy
geometric properties such as being collision-free or having
appropriate inverse kinematics relationships, for example.  To focus
the generation of objects on the most relevant parts of the space, we
allow the planner to use ``dummy'' abstract objects as long as it
plans to generate concrete values for them.  These concrete values
will be generated in the next iteration and will, hopefully,
contribute to finding a solution with all ground objects.

As before, we transform the \algname{} problem into a sequence of
\pddl{} problems, but this time we augment the planning domain with
abstract objects, two new fluents, and a new set of operator schemas.  
Let $\{\gamma_1, ..., \gamma_\theta\}$ be a set of {\em
  abstract objects} which are not assumed to satisfy any static
predicates in the initial state. 
We introduce the fluent predicate $\id{Concrete}$, which is initially
false for any object $\gamma_i$ but true for all actual ground
objects; so for all $o \in {\cal O}$, we add $\id{Concrete}(o)$ to
$s_{init}$.  The planner can ``cause'' an abstract object $\gamma_i$ to
satisfy $\id{Concrete}(\gamma_i)$ by generating it using a special
{\em stream operator}, as described below.  
We define procedure \proc{tform-ops}
that transforms each operator scheme $a(x_1, ..., x_n) \in {\cal A}$
by adding preconditions $\id{Concrete}(x_i)$ for $i = 1, ..., n$ to 
ensure that the parameters for $a$ are grounded before its
application during the search. 

To manage the balance in which streams are called,
for each stream schema $\sigma$, we introduce a new predicate
$\id{Blocked}_\sigma$; when applied to arguments
$(X_1, \ldots, X_n)$, it will temporarily prevent the use of stream
$\sigma(Y_1, ... Y_m \mid X_1, ..., X_n)$. Additionally, we add any new objects
and static atoms first to sets ${\cal O}_t$ and ${\cal S}_t$ temporarily before adding 
them to ${\cal O}$ and ${\cal S}$ to ensure any necessary existing streams
are called. Alternatively, we can immediately add directly to ${\cal O}$ and ${\cal S}$
a finite number of times before first adding to ${\cal O}_t$ and ${\cal S}_t$ and
still preserve completeness.
Let the procedure
\proc{tform-streams} convert each stream schema into an operator
schema $\sigma$ of the following form.

\begin{footnotesize}
\begin{codebox}
\Procname{\proc{StreamOperator}$_\sigma(X_1, ..., X_m, Y_1, ..., Y_n)$:}

\zi \kw{pre} = $\sigma.\kw{inp} \;\cup\;$ \=$\{\id{Concrete}(X_i) \mid i = 1,
..., m\}\; \cup$ 
\zi\>$\{\neg \id{Blocked}_\sigma(X_1, ..., X_m)\}$
\zi \kw{eff} = $\sigma.\kw{out} \;\cup \{\id{Concrete}(Y_i) \mid i = 1, ..., n\}$
\end{codebox}
\end{footnotesize}

\noindent
It allows \proc{s-plan} to explicitly plan to generate a tuple of
concrete objects from stream $\sigma(Y_1, ... Y_m \mid x_1, ..., x_n)$
as long as the $x_i$ have been made concrete and the stream instance
is not blocked.  





The procedure $\proc{\focused{}}$, shown below, implements the \focused{}
approach to planning.  It takes the same inputs as the \eager{}
version, but with the maximum number of abstract objects $\theta \geq 1$
specified as a meta-parameter, rather than $K$.  It also maintains a set ${\cal O}$ of
concrete objects and a set ${\cal S}$ of fluent and static atoms true in the
initial state.
In each iteration of the main loop, a
\strips{} planning instance is constructed: the initial state is
augmented with the set of static atoms indicating which streams are
blocked and fluents asserting that the objects in ${\cal O}$ are
concrete; the set of operator schemas is transformed as described above
and augmented with the stream operator schemas, and the set of objects
is augmented with the abstract objects.  If a plan is obtained and it
contains only operator instances, then it will have only concrete
objects, and it can be returned directly.  If the plan contains abstract objects, it also contains stream
operators, and \proc{add-objects}
is called to generate an appropriate set of new objects.  If no plan
is obtained, and if no streams are currently blocked as well as no new objects
or initial atoms have been produced since the last reset, then the problem
is proved to be infeasible.  
Otherwise, the problem is reset by unblocking all streams and adding
${\cal O}_t$ and ${\cal S}_t$ to ${\cal O}$ and ${\cal S}$, in order to allow
a new plan with abstract objects to be generated.

\begin{footnotesize}
\begin{codebox}
\Procname{$\proc{focused}((({\cal P}_s, {\cal P}_f,  {\cal C}_0, {\cal A}, {\cal X}, \Sigma), O_0, s_0, s_*), \proc{s-plan}, \theta):$}
\zi ${\cal O} =  {\cal C}_0 \cup O_0;\; {\cal S} = s_0$;\; ${\cal O}_t = {\cal S}_t = \beta_t = \beta_p = \emptyset$ 
\zi $\bar{\cal A} = \proc{tform-ops}({\cal A});\; \bar{\Sigma} = \proc{tform-streams}(\Sigma)$ 
\zi \While \kw{True}: \Do
\zi $\pi = \proc{s-plan}($\=${\cal P}_s \cup {\cal P}_f, \bar{\cal A} \cup {\cal X} \cup \bar{\Sigma}, {\cal O} \cup \{\gamma_1, ..., \gamma_\theta\},$
\zi \> $\;{\cal S} \cup \beta_t \cup \beta_p \cup \{\id{Concrete}(o \in {\cal O})\} ,s_*)$ 
\zi \If $\pi $\=$\;\not= \kw{None}$: \Then 
\zi \If $\forall a \in \pi$, $\proc{schema}(a) \in \bar{\cal A}$: \Then 
\zi \Return $\pi$
\End
\zi $\proc{add-objects}(\pi, {\cal O}_t, {\cal S}_t,  \beta_t, \beta_p, \bar{\Sigma})$ 
\zi \Else
\zi \If ${\cal O}_t = {\cal S}_t = \beta_t = \emptyset$: \Then
\zi \Return $\kw{None}$ \Comment Infeasible
\End
\zi ${\cal O} = {\cal O} \cup {\cal O}_t; {\cal S} = {\cal S} \cup {\cal S}_t$
\zi ${\cal O}_t = {\cal S}_t = \beta_t = \emptyset$ \Comment Enable all objects \& streams
\End\End\End
\end{codebox}
\end{footnotesize}

Given a plan $\pi$ that contains abstract objects, we process it
from beginning to end, to generate a collection of new objects with
appropriate conditional relationships.  Procedure \proc{add-objects}
initializes an empty binding environment and then loops through the
instances $a$ of stream operators in $\pi$.  For each stream operator
instance, we substitute concrete objects in for abstract objects, in
the input parameters, dictated by the bindings \id{bd}, and then draw
a new tuple of objects from that conditional stream.  If there is no such
tuple of objects, the stream is exhausted and it is permanently removed from
future consideration by adding the fluent
$\id{Blocked}_\sigma(\bar{o}_x)$ to the set $\beta_p$. Otherwise, the new
objects are added to ${\cal O}_t$ and appropriate new static atoms to
${\cal S}_t$.  This stream is temporarily blocked by adding fluent
$\id{Blocked}_\sigma(\bar{o}_x)$ to the set $\beta_t$, and the
bindings for abstract objects are recorded.

\begin{footnotesize}
\begin{codebox}
\Procname{$\proc{add-objects}(\pi, {\cal O}_t, {\cal S}_t, \beta_t, \beta_p, \bar{\Sigma}):$} 
\zi $\id{bd} = \{\;\}$ \Comment Empty dictionary
\zi \For $\sigma (\bar{y} \mid \bar{x}) \in \{a \mid a \in \pi$
\kw{and} $\proc{schema}(a) \in \bar{\Sigma}\}$: \Do
\zi $\bar{o}_x = \proc{apply-bindings}(\id{bd}, \bar{x})$
\zi \If $\bar{o}_x  \neq \kw{None}$: \Then 
\zi $\bar{o}_y = \proc{next}(\sigma.f(\bar{o}_x))$  
\zi \If $\bar{o}_y  \neq \kw{None}$: \Then 
\zi ${\cal O}_t = {\cal O}_t \cup \bar{o}_y; {\cal S}_t ={\cal S}_t \cup \sigma.\kw{out}((\bar{o}_x, \bar{o}_y))$ 
\zi $\beta_t = \beta_t \cup \{\id{Blocked}_\sigma(\bar{o}_x)\}$ \Comment Temporary
\zi \For $i \in \{1, ..., |\bar{y}|\}$: \Then
\zi $\id{bd}[y_i] = o_{y, i}$
\End
\zi \Else
\zi $\beta_p = \beta_p \cup \{\id{Blocked}_\sigma(\bar{o}_x)\}$ \Comment Permanent
\End\End\End
\zi \Return ${\cal O}_t, {\cal S}_t, \beta_t, \beta_p$
\end{codebox}
\end{footnotesize}


The \focused{} algorithm is similar to the lazy shortest path
algorithm for motion planning in that it determines which streams to call, or 
analogously which edges to evaluate, by repeatedly solving optimistic problems~\cite{bohlin2000path,dellin2016unifying}.
Stream operators can be given meta-costs that reflect the time overhead to 
draw elements from the stream and the likelihood the stream produces the desired values. 
For example, stream operators that use already concrete outputs can be given large meta-costs because they will only certify a desired predicate in the typically unlikely event that their generator returns objects matching the desired outputs.
A cost-sensitive planner will avoid returning
plans that require drawing elements from expensive or unnecessary streams. 
We can combine the behaviors of \eager{} and \focused{} algorithms to
eagerly call inexpensive streams and lazily call expensive streams. This can 
be seen as automatically applying some stream operators before calling \proc{s-plan}.


\begin{thm}
The \focused{} algorithm is complete.
\end{thm}







\section{Example discrete domain}

Although the specification language is domain independent, our primary
motivating examples for the application of \algname{} are 
pick-and-place problems in infinite domains.  
We start by specifying an infinite discrete pick-and-place domain as shown in figure~\ref{fig:countable_tamp}.
We purposefully describe the domain in a way that will generalize well
to continuous and high-dimensional versions of fundamentally the same
problem.  The objects in this domain include a finite set of blocks
(that can be picked up and placed), an infinite set of poses
(locations in the world) indexed by the positive integers, and an
infinite set of robot configurations (settings of the robot's physical
degrees of freedom) also indexed by the positive integers.  
In this appendix, we give a complete Python implementation of this domain in \algname{}
The static predicates in this domain include simple static types (\id{IsConf},
\id{IsPose}, \id{IsBlock}) and typical fluents (\id{HandEmpty},
\id{Holding}, \id{AtPose}, \id{AtConfig}).  In addition, atoms of the form
$\id{IsKin}(P, Q)$ describe a static relationship between an object pose
$P$ and a robot configuration $Q$:  in this simple domain, the atom is
true if and only if $P = Q$.  Finally, fluents of the form
$\id{Safe}(b', B, P)$ are true in the circumstance that: if object $B$ were
placed at pose $P$, it would not collide with object $b'$ at its
current pose.  Because the set of blocks {\cal B} is known statically
in advance, we explicitly include all the \id{Safe} conditions.
These predicate definitions enable the following operator schemas definitions:

\begin{footnotesize}
\begin{codebox}
\Procname{\proc{Move}$(Q_1, Q_2)$:}
\zi \kw{stat} = $\{\id{IsConf}(Q_1), \id{IsConf}(Q_2)\}$
\zi \kw{pre} = $\{\id{AtConf}(Q_1)\}$
\zi \kw{eff} = $\{\id{AtConf}(Q_2), \neg \id{AtConf}(Q_1)\}$
\end{codebox}
\begin{codebox}
\Procname{\proc{Pick}$(B, P, Q)$:}
\zi \kw{stat} = $\{\id{IsBlock}(B), \id{IsPose}(P), \id{IsConf}(Q), \id{IsKin(P, Q)}\}$
\zi \kw{pre} = $\{\id{AtPose}(B, P), \id{HandEmpty}(), \id{AtConfig}(Q)\}$
\zi \kw{eff} = $\{\id{Holding}(B), \neg \id{AtPose}(B, P), \neg \id{HandEmpty}()\}$
\end{codebox}
\begin{codebox}
\Procname{\proc{Place}$(B, P, Q)$:}
\zi \kw{stat} = $\{\id{IsBlock}(B), \id{IsPose}(P), \id{IsConf}(Q), \id{IsKin(P, Q)}\}$
\zi \kw{pre} = $\{\id{Holding}(B), \id{AtConfig}(Q)\} \cup \{\id{Safe}(b' \in {\cal B}, B, P)\}$
\zi \kw{eff} = $\{\id{AtPose}(B, P), \id{HandEmpty}(), \neg \id{Holding}(B)\}$
\end{codebox}
\end{footnotesize}

We use the following axioms to evaluate the \id{Safe} predicate.
We need two slightly different definitions to handle the cases where the block $B_1$ is placed at a pose,
and where it is in the robot's hand.
The \id{Safe} axioms mention each block independently which allows us to compactly perform collision checking. 
Without using axioms, \proc{Place} would require a parameter for the pose of each block in ${\cal B}$, resulting in an prohibitively large grounded problem.

\begin{footnotesize}
\begin{codebox}
\Procname{\proc{SafeAxiom}$(B_1, P_1, B_2, P_2)$:}
\zi \kw{stat} = $\{\id{IsBlock}(B_1), \id{IsPose}(P_1), \id{IsBlock}(B_2),$ \Indentmore
\zi $\id{IsPose}(P_2), \id{IsCollisionFree}(B_1, P_1, B_2, P_2)\}$ \End
\zi \kw{pre} = $\{\id{AtPose}(B_1, P_1)\}$
\zi \kw{eff} =$\{\id{Safe}(B_1, B_2, P_2)\}$
\end{codebox}
\begin{codebox}
\Procname{\proc{SafeAxiomH}$(B_1, B_2, P_2)$:}
\zi \kw{stat} = $\{\id{IsBlock}(B_1), \id{IsBlock}(B_2), \id{IsPose}(P_2)\}$ \End
\zi \kw{pre} = $\{\id{Holding}(B_1)\}$
\zi \kw{eff} =$\{\id{Safe}(B_1, B_2, P_2)\}$
\end{codebox}
\end{footnotesize}

\paragraph{Discrete stream specification}

Next, we provide stream definitions.  The simplest stream is an
unconditional generator of poses, which are represented as objects
$\proc{Pose}(i)$ and satisfy the static predicate $\id{IsPose}$.

\begin{footnotesize}
\begin{codebox}
\Procname{\proc{Pose-U}$(P \mid ())$:}
\zi \kw{gen} = $\kw{lambda} (): \langle (\proc{Pose}(i)) \;\kw{for}\; i = 0, 1, 2 ... \rangle$
\zi \kw{inp} = $\emptyset$
\zi \kw{out} = $\{\id{IsPose}(P)\}$
\end{codebox}
\end{footnotesize}

The conditional stream \proc{CFree-T} is a test, calling the
underlying function $\proc{collide}(B_1, P_1, B_2, P_2)$;  the stream
is empty if block $B_1$ at pose $P_1$ collides with block $B_2$ at
pose $P_2$, and contains the single element $(\;)$ if it does not
collide.  It is used to certify that the tuple $(B_1, P_1, B_2, P_2)$
statically satisfies the \id{IsCollisionFree} predicate.

\begin{footnotesize}
\begin{codebox}
\Procname{\proc{CFree-T}$(() \mid B_1, P_1, B_2, P_2)$:}
\zi \kw{gen} = $\kw{lambda} (B_1, P_1, B_2, P_2):$ 
\zi \>\>\> $\langle () \;\kw{if} \;\kw{not}\; \proc{collide}(B_1, P_1, B_2, P_2) \rangle$
\zi \kw{inp} = $\{\id{IsBlock}(B_1), \id{IsPose}(P_1), \id{IsBlock}(B_2), \id{IsPose}(P_2)\}$
\zi \kw{out} = $\{\id{IsCollisionFree}(B_1, P_1, B_2, P_2)\}$
\end{codebox}
\end{footnotesize}

When we have a static relation on more than one variable, such as
\id{IsKin}, we have to make modeling choices when defining streams that
certify it.

We will consider three formulations of streams that certify \id{IsKin}
and compare them in terms of their effectiveness in a simple countable pick-and-place
problem requiring the robot gripper to pick block $A$ at a distant
initial pose $p_0 >> 1$, shown in figure~\ref{fig:countable_tamp}.

\proc{Kin-U} specifies an unconditional stream on block poses and
robot configurations;  it has no difficulty certifying the \id{IsKin}
relation between the two output variables, but it has no good way of
producing configurations that are appropriate for poses that are
mentioned in the initial state or goal.  

\begin{footnotesize}
\begin{codebox}
\Procname{\proc{Kin-U}$(P, Q \mid ())$:}
\zi \kw{gen} = $\kw{lambda} (): \langle (\proc{Pose}(i), \proc{Conf}(i)) \;\For\; \text{i} = 1, 2, ... \rangle$
\zi \kw{inp} $ = \emptyset$
\zi \kw{out} = $\{\id{IsPose}(P), \id{IsConf}(Q), \id{IsKin}(P, Q)\}$
\end{codebox}
\end{footnotesize}

\proc{Kin-T} specifies a test stream that can be used, together with
the \proc{Pose-U} stream and an analogous stream for generic
configurations to produce certified kinematic pairs $P, Q$.  This is
an encoding of a ``generate-and-test'' strategy, which may be highly
inefficient, relying on luck that the pose generator and the
configuration generator will independently produce values that have
the appropriate relationship.

\begin{footnotesize}
\begin{codebox}
\Procname{\proc{Kin-T}$(() \mid P, Q)$:}
\zi \kw{gen} = $\kw{lambda} (P, Q): \langle () \;\kw{if}\; Q = \proc{inverse-kin}(P)\rangle$
\zi \kw{inp} = $\{\id{IsPose}(P), \id{IsConf}(Q)\}$
\zi \kw{out} = $\{\id{IsKin}(P, Q)\}$
\end{codebox}
\end{footnotesize}

Finally, \proc{Kin-C} specifies a conditional stream, which takes a
pose $P$ as input and generates a stream of configurations (in this
very simple case, containing a single element) certified to satisfy
the \id{IsKin} relation.  It relies on an underlying function
$\proc{inverse-kin}(p)$ to produce an appropriate robot configuration
given a block pose.

\begin{footnotesize}
\begin{codebox}
\Procname{\proc{Kin-C}$(Q \mid P)$:}
\zi \kw{gen} = $\kw{lambda} (P): \langle (\proc{inverse-kin}(P)) \rangle$
\zi \kw{inp} = $\{\id{IsPose}(P)\}$
\zi \kw{out} = $\{\id{IsConf}(Q), \id{IsIK}(P, Q)\}$
\end{codebox}
\end{footnotesize}

In our example domain, both \proc{Kin-U} and \proc{Kin-T} require the
enumeration of poses and configurations $(p_i, q_i)$ from $i = 0, 1,
..., p_0$ before certifying $\id{IsIK}(p_*, q_*)$, allowing \strips{}
to make a plan include the operator
$\proc{Pick}(A, p_*, q_*)$. Moreover, \proc{Kin-T} will test all pairs
of configurations and poses.  In contrast, \proc{Kin-C}  can produce
$q$ directly from $p_0$ without enumerating any other poses or
configurations. The conditional formulation is advantageous
because it produces a paired inverse kinematics configuration quickly
and without substantially expanding the size of the problem. 

Table~\ref{fig:countable_table} validates this intuition though an
experiment comparing these stream specifications.  The initial pose of
the object $p_0$ is chosen from $1, 100, 1000$.  All trials have a timeout of 120 seconds and use the
\eager{} algorithm with $K=1$ implemented in Python.
As predicted, the \proc{Kin-U} and \proc{Kin-T}
streams require many more calls than \proc{Kin-C} as $p_0$ increases
and lead to substantially longer runtimes for a very simple problem.


\begin{figure}[ht]
\centering
\includegraphics[width=0.23\textwidth]{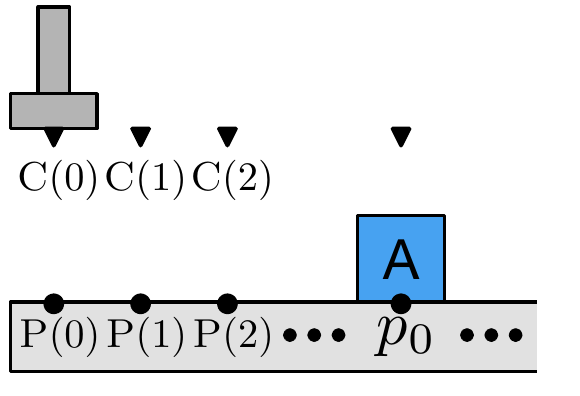}
\includegraphics[width=0.23\textwidth]{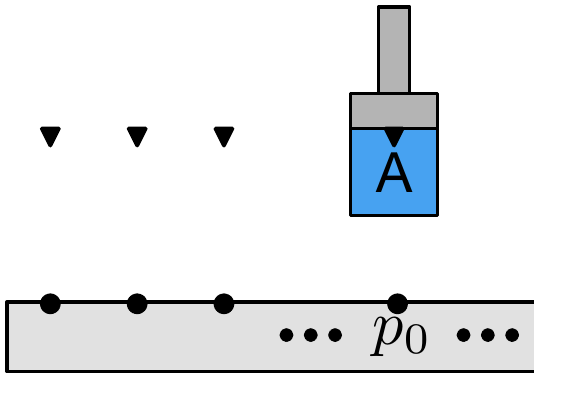}
\caption{The initial state and goal state in an infinite, discrete pick-and-place problem requiring picking block A. } 
\label{fig:countable_tamp}
\end{figure}

\begin{table}[t]
\begin{footnotesize}
\begin{tabular}{||c||g|c|c||g|c|c||g|c|c||}
\hline
$p_0$ & \multicolumn{3}{|c||}{\proc{Kin-U}}&\multicolumn{3}{|c||}{\proc{Kin-T}}&\multicolumn{3}{|c||}{\proc{Kin-C}}\\
\hline
&
t & i & c &
t & i & c &
t & i & c \\
\hline
1 & 
.1 & 3 & 2 &
.2 & 6 & 9 &
.1 & 3 & 2 \\
\hline
100 & 
29 & 102 & 101 &
71 & 303 & 10360 &
.1 & 3 & 2 \\
\hline
1000 & 
- & 180 & 179 &
- & 381 & 16383 &
.1 & 3 & 2 \\
\hline
\end{tabular}
\end{footnotesize}
\caption{The runtime (t), number of search iterations (i), and number of
generator calls (c) for the countable pick-and-place \proc{Kin} stream representation experiment.}
\label{table}
\label{fig:countable_table}
\end{table}

\section{Continuous domains}

The \algname{} approach can be applied directly in continuous domains such as the problem in figure~\ref{fig:obstruction}.
In this case, the streams will have to generate samples from sets of continuous
dimensions, and the way that samples are generated may have a
significant impact on the efficiency and completeness of the approach
with respect to the domain problem.  (Note that the \algname{} planing
algorithms are complete with respect to the streams of enumerated
values they are given, but if these value streams are not, in some
sense, complete with respect to the underlying problem domain, then
the resulting combined system may not be complete with respect to the
original problem.)  Samplers that produce a {\em dense}
sequence~\cite{Lavalle06} are good candidates for stream generation.

\subsection{Continuous stream specification}

With some minor modifications, we can extend our discrete
pick-and-place domain to a bounded interval $[0, L]$ of the real
line.  Poses and configurations are now continuous objects
$p, q \in [0, L]$ from an uncountably infinite domain.  
The stream \proc{Pose-U} now has a generator that samples $[0, L]$
uniformly at random.  

While in the discrete case the choice of streams just affected the
size of the problem, in the continuous case, the choice of streams can
affect the feasibility of the problem. 
In the continuous simple pick-and-place domain, suppose that the blocks have width 1 and the
gripper has width $\delta \geq 1$.  A kinematics pair $(p,q)$ is valid
if and only if the gripper is entirely over the block, i.e.,
$p + 1/2 \leq q + \delta/2$ and $p - 1/2 \geq q -\delta/2$.  Consider
the case where \proc{Kin-U} and \proc{Kin-C} are implemented using
random samplers.  \proc{Kin-U} will almost certainly generate a
sequence of infeasible \strips{} problems, because the probability
that the point $p_0$ is produced from its generator is zero.  For
$\delta > 1$, the configuration stream has nonzero probability of
generating a $q$ that would constitute a valid kinematics pair with
$p$ as certified by \proc{Kin-T}. But this probability can be made
arbitrary small as $\delta \to 1$. Only the \proc{Kin-C} strategy
is robust to the choice of $\delta$. Table~\ref{fig:continuous_table}
shows the results of an experiment analogous to the one in
table~\ref{fig:countable_table}, but which varies
$\delta \in \{1.5, 1.01\}$ instead of varying $p_0$. \proc{Kin-U} was unable
to solve either problem and \proc{Kin-T} could not find a solution in
under two minutes for $\delta = 1.01$. But once again, the conditional
formulation using \proc{Kin-C} performs equivalently for different
values of $\delta$.

\begin{table}[t]
\begin{footnotesize}
\begin{tabular}{||c||g|c|c||g|c|c||g|c|c||}
\hline
$\delta$ & \multicolumn{3}{|c||}{\proc{Kin-U}}&\multicolumn{3}{|c||}{\proc{Kin-T}}&\multicolumn{3}{|c||}{\proc{Kin-C}}\\
\hline
& 
t & i & c &
t & i & c &
t & i & c \\
\hline
1.5 & 
- & 191 & 190 &
3.1 & 75 & 745 &
.1 & 2 & 1 \\
\hline
1.01 & 
- & 181 & 180 &
- & 297 & 18768 &
.1 & 2 & 1 \\
\hline
\end{tabular}
\end{footnotesize}
\caption{The runtime (t), number of search iterations (i), and number of generator calls (c) for the continuous pick-and-place \proc{Kin} stream representation experiment.}
\label{table}
\label{fig:continuous_table}
\end{table}

\subsection{Focused algorithm example}

\begin{figure*}[ht]
\centering
\includegraphics[width=0.23\textwidth]{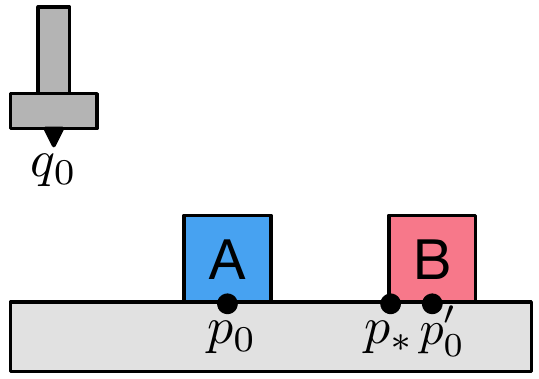}
\includegraphics[width=0.23\textwidth]{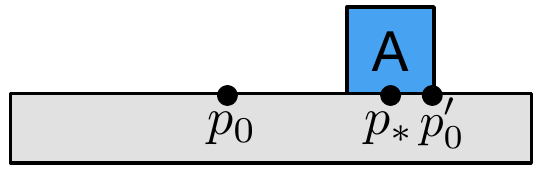}
\includegraphics[width=0.23\textwidth]{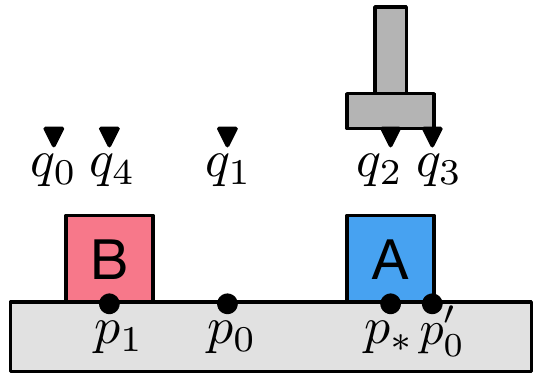}
\caption{Initial state for countable pick-and-place problem requiring
  picking and placing block A, with a single obstacle.} 
\label{fig:obstruction}
\end{figure*}

The previous examples investigated the effect of different representational choices on the tractability and even feasibility of the resulting \algname{} problem. 



The example in figure~\ref{fig:obstruction} illustrates the behavior
of the focused algorithm on continuous a pick-and-place problem with the
goal condition that block $A$ is at pose $p_*$. Because block A, when
at $p_*$, collides with block B at its initial pose $p_0'$, solving
this problem requires
moving block $B$ out of the way to place block $A$. Suppose we use
\proc{Kin-C} to model the problem.  We will omit \proc{Move} operators
for the sake of clarity,  and use capital letters to denote abstract
objects. On the first iteration, the \focused{} algorithm will produce the
following plan (possibly ordered slightly differently):

\begin{footnotesize}
\begin{align*}
\pi_1 &=  \big(\proc{Kin-C}(Q_1 \mid p_0), \proc{Pick}(A, p_0, Q_1), \proc{Kin-C}(Q_2 \mid p_*), \\
& \proc{CFree-T}(() \mid B, p_0', A, p_*), \proc{Place}(A, p_*, Q_2)\big)
\end{align*}
\end{footnotesize}

The generation of values proceeds as follows.
$\proc{Kin-C}(Q_1 \mid p_0)$ will produce $Q_1 \leftarrow
q_1$. $\proc{Kin-C}(Q_2 \mid p_*)$ will produce $Q_2 \leftarrow
q_2$. However, $\proc{CFree-T}(() \mid B, p_0', A, p_0)$ will produce the
empty stream because $p_0'$ collides with $p_*$. Thus, the plan
$\pi_1$ definitively cannot be completed. The algorithm adds $q_1$ and
$q_2$ to the current \pddl{} problem and records the failure of
$\proc{CFree-T}(() \mid B, p_0', A, p_0)$. 
On the next iteration, the \focused{} algorithm will produce the
following plan. 

\begin{footnotesize}
\begin{align*}
\pi_2 &=  \big(\proc{Kin-C}(Q_1 \mid p_0'), \proc{Pick}(B, p_0', Q_1), \proc{Pose-U}(P_1 \mid ()), \\
& \proc{Kin-C}(Q_2 \mid P_1), \proc{CFree-T}(() \mid A, p_0, B, P_1), \proc{Place}(B, P_1, Q_2), \\
& \proc{Pick}(A, p_0, q_1), \proc{CFree-T}(() \mid B, P_1, A, p_*), \proc{Place}(A, p_*, q_2) \big)
\end{align*}
\end{footnotesize}

The generation of values proceeds as follows.
$\proc{Kin-C}(Q_1 \mid p_0')$ will produce $Q_1 \leftarrow
q_3$. $\proc{Pose-U}(P_1 \mid ())$ will produce $P_1 \leftarrow p_1$.
$\proc{Kin-C}(Q_2 \mid p_1)$ will produce $Q_2 \leftarrow q_4$.  Let's
assume that $P_1 \leftarrow p_1$ is randomly sampled and turns out to
not be in collision with $p_*$. If $p_1$ turned out to be in collision
with $p_*$, the next iteration would first fail once, then repeat this
process on the next iteration to generate a new $P_1$.  So,
$\proc{CFree-T}(() \mid A, p_1, B, p_0)$ will produce the stream
$\langle () \rangle$ indicating that $p_1$ and $p_0$ are not in
collision.  Thus, all of the properties have been successfully
satisfied, so the following plan is a solution.
It is critical to note that, for example, had there been several other
pose constants appearing in the initial state, \focused{} would never have
found inverse kinematic solutions for them:
because the planner guides the sampling, only stream elements that
play a direct role in a plausible plan are generated.

\begin{footnotesize}
\begin{align*}
\pi_* &=  \big(\proc{Pick}(B, p_0', q_3), \proc{Place}(B, p_1, q_4), \proc{Pick}(A, p_0, q_1), \\
& \proc{Place}(A, p_*, q_2)\big)
\end{align*}
\end{footnotesize}
\section{Realistic robot domain}


\begin{table*}
\centering
\begin{footnotesize}
\begin{tabular}{||c||c|g|c|c||c|g|c|c||c|g|c|c||}
\hline
$\Pi$ & \multicolumn{4}{|c||}{\eager{}, $K=1$}&\multicolumn{4}{|c||}{\eager{}, $K=100$}&\multicolumn{4}{|c||}{\focused{}}\\
\hline
&
\% & t & i & c &
\% & t & i & c &
\% & t & i & c \\
\hline
1 & 
88 & 2 & 23 & 268 &
68 & 5 & 2 & 751 &
84 & 11 & 6 & 129 \\
\hline
2-0 & 
100 & 23 & 85 & 1757 &
100 & 9 & 3 & 2270 &
100 & 2 & 3 & 180 \\
\hline
2-8 & 
0 & - & - & - &
100 & 55 & 5 & 17217 &
100 & 7 & 3 & 352 \\
\hline
2-16 & 
0 & - & - & - &
100 & 112 & 6 & 36580 &
100 & 19 & 3 & 506 \\
\hline
\end{tabular}
\end{footnotesize}
\caption{The success percentage (\%), runtime (t), search iterations (i), and
number of stream calls (c) for the high-dimensional task and motion planning experiments.}
\label{table:results}
\end{table*}


Finally, we extend our continuous pick-and-place to the
high-dimensional setting of a robot operating in household-like
environments.  Poses of physical blocks are 6-dimensional and robot
configurations are 11-dimensional.  We introduce two new object types:
grasps and trajectories.  Each block has a set of 6D relative grasp
transforms at 
which it can be grasped by the robot.  Trajectories are
finite sequences of configuration waypoints which must be included in
collision checking.
The extended \proc{Pick} operator, \proc{CFree-T} test and \proc{Kin-C}
stream templates are:

\begin{footnotesize}
\begin{codebox}
\Procname{\proc{Pick}$(B, P, G, Q, T)$:}
\zi \kw{stat} = $\{\id{IsBlock}(B), ..., \id{IsTraj}(T), \id{IsKin(P, G, Q, T)}\}$
\zi \kw{pre} = $\{\id{AtPose}(B, P), \id{HandEmpty}(), \id{AtConfig}(Q)\} \;\cup$ \Indentmore
\zi $\{\id{Safe}(b', B, G, T) \mid b' \in {\cal B}\}$ \End
\zi \kw{eff} = $\{\id{Holding}(B, G), \neg \id{AtPose}(B, P), \neg \id{HandEmpty}()\}$
\end{codebox}
\begin{codebox}
\Procname{\proc{CFree-T}$(() \mid B_1, P_1, B_2, G, T)$:}
\zi \kw{gen} = $\kw{lambda} (B_1, P_1, B_2, G, T):$ 
\zi \>\>\> $\langle () \;\kw{if} \;\kw{not}\; \proc{collide}(B_1, P_1, B_2, G, T) \rangle$
\zi \kw{inp} = $\{\id{IsBlock}(B_1),...,  \id{IsTraj}(T)\}$
\zi \kw{out} = $\{\id{IsCollisionFree}(B_1, P_1, B_2, G, T)\}$
\end{codebox}
\begin{codebox}
\Procname{\proc{Kin-C}$(Q, T \mid P, G)$:}
\zi \kw{gen} = $\kw{lambda} (P): \langle (Q, T) \mid Q \sim \proc{inverse-kin}(PG^{-1}), $\Indentmore
\zi $T \sim \proc{motions}(q_{rest}, Q) \rangle$ \End
\zi \kw{inp} = $\{\id{IsPose}(P), \id{IsGrasp}(G)\}$
\zi \kw{out} = $\{\id{IsKin}(P, G, Q, T), \id{IsConf}(Q), \id{IsTraj}(T)\}$
\end{codebox}
\end{footnotesize}

\proc{Pick} adds grasp $G$ and trajectory $T$ as parameters and includes
$\id{Safe}(b', B, G, T)$ preconditions to verify that $T$ while holding
$B$ at grasp $G$ is safe with respect to each other block $b'$.
$\id{Safe}(b', B, G, T)$ is updated using \proc{SafeAxiom} which has a
$\id{IsCollisionFree}(B_1, P_1, B_2, G, T)$ static precondition.
Here, a collision check for block $B_1$ at pose $P_1$ is performed for
each configuration in $T$. Instead of simple
blocks, physical objects in this domain are general unions of convex
polygons. Although checking collisions here is more
complication than in 1D, it can be treated in the same way, as an
external function.

The \proc{Kin} streams must first produce a grasp configuration $Q$
that reaches manipulator transform $PQ^{-1}$ using \proc{inverse-kin}. Additionally, they include 
a motion planner \proc{motions} to generate legal trajectory values $T$ from a constant rest configuration $q_{rest}$
to the grasping configuration $Q$ that do not
collide with the fixed environment. In this domain, the procedures for collision checking and 
finding kinematic solutions are significantly more involved and
computationally expensive than in the previous domains, but their
underlying function is the same.

\subsection{Experiments}

We applied the \eager{} and \focused{} algorithms on four challenging
pick-and-place problems to demonstrate that a general-purpose
representation and algorithms can be used to achieve good performance
in difficult problems.  For both algorithms, test streams are
always evaluated as soon as they are instantiated.  
We experimented on two domains shown in
figure~\ref{fig:tamp_domains}, which are similar to problems
introduced by~\cite{GarrettIROS15}. The first domain, in which problem 1 is
defined, has goal conditions that the green object be in the right bin
and the blue object remain at its initial pose. This requires the
robot to not only move and replace the blue block but also to place
the green object in order to find a new grasp to insert it into the
bin.  The second domain, in which problems 2-0, 2-8, and 2-16 are
defined, requires moving an object out of the way and placing the
green object in the green region. For problem 2-$n$ where
$n \in \{0, 8, 16\}$ there are $n$ other blocks on a separate table
that serve as distractors. The streams were implemented using
the OpenRAVE robotics framework~\cite{openrave}.
A Python implementation of \algname{} can be found here:
\texttt{https://github.com/caelan/stripstream}.

\begin{figure}[h]
\centering
\includegraphics[width=0.30\textwidth]{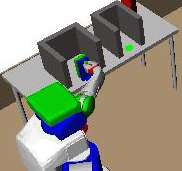}
\caption{Problem 1.} 
\label{fig:tamp_domains}
\end{figure}

The results compare the \eager{} algorithm where $K=1$ and $K=100$
with the \focused{} algorithm.  Table~\ref{table:results} shows the
results of 25 trials, each with a timeout of 120 seconds.
The \eager{} algorithms result in
significantly more stream calls than the focused algorithm.
These calls can significantly increase the total runtime
because each inverse kinematic and collision primitive itself is expensive.
Additionally, the \eager{} algorithms are significantly affected by
the increased number of distractors, making them unsuitable for
complex real-world environments. The \focused{} algorithm, however, is
able to selectively choose which streams to call resulting in
significantly better performance in these environments.

\section{Conclusion}

The \algname{} problem specification formalism can be used to describe
a large class of planning problems in infinite domains and provides a
clear and clean interface to problem-specific sampling methods in
continuous domains.  
The \eager{} and, in particular, \focused{} planning algorithms
take advantage of the specification to provide efficient solutions to
difficult problems.

\newpage
\bibliographystyle{named}
\bibliography{references}

\newpage

\section{Appendix}


\begin{thm}  \label{thm:re}
${\cal O}_\Pi$ and ${\cal S}_\Pi$ are recursively enumerable (RE). 
\end{thm}
\begin{proof}
Consider an enumeration procedure for ${\cal O}_\Pi$ and ${\cal S}_\Pi$:
\begin{itemize}
\item The first sequences of elements in ${\cal O}_\Pi$ and 
${\cal S}_\Pi$ are ${\cal C}_0 \cup O_0$ and $s_0$ respectively.
\item Initialize a set of stream instances $\Sigma_\Pi = \emptyset$.
\item Repeat:
\begin{itemize}
\item For each stream schema $\sigma \in \Sigma$, add all
  instantiations $\sigma(\bar{Y} \mid \bar{x})$ where $x \subseteq {\cal O}_\Pi$ such that
  $\sigma.\kw{inp}(\bar{x})$ is contained within ${\cal S}_\Pi$,
  to $\Sigma_\Pi$. There are finitely many new elements of
  $\Sigma_\Pi$.
\item For each stream instance $\sigma(\bar{Y} \mid \bar{x}) \in \Sigma_\Pi$, add
  $\bar{y} = \proc{next}(\sigma.f(\bar{x}))$ to ${\cal O}_\Pi$ and add
  $\sigma^*.\kw{out}((\bar{x}, \bar{y}))$ to ${\cal S}_\Pi$.
  There are finitely many new elements of ${\cal O}_\Pi$ and ${\cal S}_\Pi$.
\end{itemize}
\end{itemize}
This procedure will enumerate all possible objects and all possible
initial atoms generated within the problem $\Pi$.
\end{proof}

\begin{thm}
The existence of a solution for a \algname{} problem $\Pi$ is
undecidable. 
\end{thm}
\begin{proof}
We use a reduction from the halting problem.
Given a Turning machine TM, we construct a \algname{} 
problem $\Pi_{TM}$ with a single operator
\proc{Halt}(X) with $\kw{stat} = \{\id{IsReachable}(X)\}$, 
$\kw{pre} = \emptyset$, and $\kw{eff} = \{\id{Reached}(X)\}$ where
$\id{IsReachable}$ and $\id{Reached}$ are a static and fluent predicate
defined on TM's states. There is a single unconditional stream $\proc{Reachable-U}(X \mid ())$ which 
enumerates the states of TM by simulating one step of TM upon each call.
Let $s_0 = \emptyset$ and $s_* = \{\id{Reached}(a)\}$
where $a$ is the accept state for TM. $\Pi_{TM}$ has a solution if and only if TM halts.
Thus, \algname{} is undecidable.
\end{proof}

\begin{thm}
The existence of a solution for a \algname{} problem $\Pi$ is
semi-decidable. 
\end{thm}
\begin{proof}
From the recursive enumeration of ${\cal O}_\Pi$ and ${\cal S}_\Pi$ we produce a
recursive enumeration of finite planning problems. Planning problem
$i$ is grounded using all objects and static atoms enumerated up through element $i$.
Plan existence in a finite universe is decidable. 
Thus, for feasible problems, applying a finite decision procedure to the sequence of finite planning problems
will eventually reach a planning problem for which a plan exists and
produce it.
\end{proof}

\begin{thm}
The \eager{} algorithm is complete.
\end{thm}
\begin{proof}
The \eager{} algorithm constructs ${\cal O}$ and ${\cal S}$ in the same way 
as theorem~\ref{thm:re} for ${\cal O}_\Pi$ and ${\cal S}_\Pi$ 
except that it 
calls \proc{next} in batches of $K$. 
Thus, any finite subsets of 
${\cal O}_\Pi$ and ${\cal S}_\Pi$ will be included in ${\cal O}$ and ${\cal S}$
after a finite number of iterations.
Let $\pi_*$ be a solution to a feasible \algname{} problem $\Pi$.
Consider the first iteration where ${\cal O}$ and ${\cal S}$ contain the set of objects used along $\pi_*$ and static atoms supporting $\pi_*$.
On that iteration, \proc{s-plan} will return some solution (if not $\pi_*$) in finite time because it is sound and complete.
\end{proof}

\begin{thm}
The \focused{} algorithm is complete.
\begin{proof}

Define an episode as the \focused{} algorithm iterations between the last reset (${\cal O}_t = {\cal S}_t = \beta_t = \emptyset$) and the next reset. 
Consider a minimum length solution $\pi_*$ to a feasible \algname{} problem $\Pi$. Let ${\cal O}_* \subseteq {\cal O}_\Pi$ be the set of objects used along $\pi_*$ and ${\cal S}_* \subseteq {\cal S}_\Pi$ be the set of static atoms supporting $\pi_*$.

For each episode, consider the following argument.
By theorem~\ref{thm:re}, there exists a sequence of stream instance calls which produces ${\cal O}_*$ and ${\cal S}_*$ from the current ${\cal O}$ and ${\cal S}$. Let $\Sigma_*$ be the minimum length sequence that satisfies this property.
$\Sigma_*$ may include the same stream instance several times if multiple calls are needed to produce the necessary values.
On each iteration, \proc{\focused{}} creates a finite \strips{} problem $\tilde{\Pi}$ by augmenting $\Pi$ with the abstract objects $\{\gamma_1, ..., \gamma_\theta\}$ and the stream operators $\bar{\Sigma}$. 
Because ${\cal O}_t$ and ${\cal S}_t$ are withheld, ${\cal O}$ and ${\cal S}$ are fixed for all iterations within the episode. 
Thus, a finite number of simple plans are solutions for $\tilde{\Pi}$.
One of these plans, $\tilde{\pi}_*$, is $\Sigma_*$ concatenated with $\pi_*$ where additionally any object $o \in ({\cal O}_* \setminus {\cal O})$ is replaced with some abstract object $\gamma$. Assume all redundant stream operators are removed from $\tilde{\pi}_*$.
The same $\gamma$ can stand in for several $o$ on $\tilde{\pi}_*$ at once.
Thus, \proc{\focused{}} will be complete for any $\theta \geq 1$.

We will show that at least one 
stream instance in $\Sigma_*$ will called performed during each episode.
On each iteration, \proc{s-plan} will identify a plan $\pi$. 
If $\pi$ does not involve any abstract objects and is fully supported, it is a solution. 
Otherwise, \proc{add-objects} will call each stream $\sigma(\bar{Y} \mid \bar{o}_x)$ associated with $\pi$. It adds $\id{Blocked}_\sigma(\bar{o}_x)$ to $\beta_t$, preventing $\pi$ and all other plans using $\sigma(\bar{Y} \mid \bar{o}_x)$ from being re-identified within this episode. 
If $\pi$ overlaps with $\tilde{\pi}_*$ and $\sigma(\bar{Y} \mid \bar{o}_x) \in \Sigma_*$, then the episode has succeeded.
Otherwise, this process repeats on the next iteration. 
Eventually a stream instance in $\Sigma_*$ will be called, or $\tilde{\pi}_*$ itself will be the only remaining unblocked plan for $\tilde{\Pi}$. In which case, \proc{s-plan} will return $\tilde{\pi}_*$, and \proc{add-objects} will call a stream instance in $\Sigma_*$.

$|\Sigma_*|$ strictly decreases after each episode. Inductively applying this, after a finite number of episodes, 
${\cal O}_* \subseteq {\cal O}$ and ${\cal S}_* \subseteq {\cal S}$. During the next episode, \proc{s-plan} will be guaranteed to return some solution (if not $\pi_*$).
\end{proof}
\end{thm}

%

\subsection{Python example discrete domain}

Figure~\ref{fig:python} gives a complete encoding of the example discrete domain and a problem instance within it using our Python implementation of \algname{}. In the specified problem instance, the initial state consists of three blocks placed in a row. The goal is to shift each of the blocks over one pose. The Python syntax of \algname{} intentionally resembles the Planning Domain Definition Language (\pddl{})~\cite{mcdermott1998pddl}. We use several common features of \pddl{} that extend \strips{}. The resulting encoding is equivalent to previously described \strips{} formulation but is more compact. 
We use object types \proc{BLOCK}, \proc{POSE}, \proc{CONF} instead of static predicates \id{IsBlock}, \id{IsPose}, and \id{IsConf}.
Additionally, we use several Action Description Language ({\sc ADL}) logical operations including \proc{Or}, \proc{Equal}, \proc{ForAll}, and \proc{Exists}. The universal quantifier (\proc{ForAll}) is over \proc{BLOCK}, a finite type, and thus is a finite conjunction. 

\begin{figure*}
\begin{footnotesize}
\begin{Verbatim}[commandchars=\\\{\}]
\PYG{k+kn}{from} \PYG{n+nn}{stripstream} \PYG{k+kn}{import} \PYG{n}{Type}\PYG{p}{,} \PYG{n}{Param}\PYG{p}{,} \PYG{n}{Pred}\PYG{p}{,} \PYG{n}{Not}\PYG{p}{,} \PYG{n}{Or}\PYG{p}{,} \PYG{n}{And}\PYG{p}{,} \PYG{n}{Equal}\PYG{p}{,} \PYG{n}{Exists}\PYG{p}{,} \PYGZbs{}
  \PYG{n}{ForAll}\PYG{p}{,} \PYG{n}{Action}\PYG{p}{,} \PYG{n}{Axiom}\PYG{p}{,} \PYG{n}{GeneratorStream}\PYG{p}{,} \PYG{n}{TestStream}\PYG{p}{,} \PYG{n}{STRIPStreamProblem}

\PYG{n}{blocks} \PYG{o}{=} \PYG{p}{[}\PYG{l+s+s1}{\PYGZsq{}block}\PYG{l+s+si}{\PYGZpc{}i}\PYG{l+s+s1}{\PYGZsq{}}\PYG{o}{\PYGZpc{}}\PYG{n}{i} \PYG{k}{for} \PYG{n}{i} \PYG{o+ow}{in} \PYG{n+nb}{range}\PYG{p}{(}\PYG{l+m+mi}{3}\PYG{p}{)]}
\PYG{n}{num\PYGZus{}poses} \PYG{o}{=} \PYG{n+nb}{pow}\PYG{p}{(}\PYG{l+m+mi}{10}\PYG{p}{,} \PYG{l+m+mi}{10}\PYG{p}{)} \PYG{c+c1}{\PYGZsh{} a very large number of poses}
\PYG{n}{initial\PYGZus{}config} \PYG{o}{=} \PYG{l+m+mi}{0} \PYG{c+c1}{\PYGZsh{} initial robot configuration is 0}
\PYG{n}{initial\PYGZus{}poses} \PYG{o}{=} \PYG{p}{\PYGZob{}}\PYG{n}{block}\PYG{p}{:} \PYG{n}{i} \PYG{k}{for} \PYG{n}{i}\PYG{p}{,} \PYG{n}{block} \PYG{o+ow}{in} \PYG{n+nb}{enumerate}\PYG{p}{(}\PYG{n}{blocks}\PYG{p}{)\PYGZcb{}} \PYG{c+c1}{\PYGZsh{} initial pose for block i is i}
\PYG{n}{goal\PYGZus{}poses} \PYG{o}{=} \PYG{p}{\PYGZob{}}\PYG{n}{block}\PYG{p}{:} \PYG{n}{i}\PYG{o}{+}\PYG{l+m+mi}{1} \PYG{k}{for} \PYG{n}{i}\PYG{p}{,} \PYG{n}{block} \PYG{o+ow}{in} \PYG{n+nb}{enumerate}\PYG{p}{(}\PYG{n}{blocks}\PYG{p}{)\PYGZcb{}} \PYG{c+c1}{\PYGZsh{} goal pose for block i is i+1}

\PYG{n}{BLOCK}\PYG{p}{,} \PYG{n}{POSE}\PYG{p}{,} \PYG{n}{CONF} \PYG{o}{=} \PYG{n}{Type}\PYG{p}{(),} \PYG{n}{Type}\PYG{p}{(),} \PYG{n}{Type}\PYG{p}{()} \PYG{c+c1}{\PYGZsh{} Object types}
\PYG{n}{B1}\PYG{p}{,} \PYG{n}{B2} \PYG{o}{=} \PYG{n}{Param}\PYG{p}{(}\PYG{n}{BLOCK}\PYG{p}{),} \PYG{n}{Param}\PYG{p}{(}\PYG{n}{BLOCK}\PYG{p}{)} \PYG{c+c1}{\PYGZsh{} Free parameters}
\PYG{n}{P1}\PYG{p}{,} \PYG{n}{P2} \PYG{o}{=} \PYG{n}{Param}\PYG{p}{(}\PYG{n}{POSE}\PYG{p}{),} \PYG{n}{Param}\PYG{p}{(}\PYG{n}{POSE}\PYG{p}{)}
\PYG{n}{Q1}\PYG{p}{,} \PYG{n}{Q2} \PYG{o}{=} \PYG{n}{Param}\PYG{p}{(}\PYG{n}{CONF}\PYG{p}{),} \PYG{n}{Param}\PYG{p}{(}\PYG{n}{CONF}\PYG{p}{)}

\PYG{n}{AtConf} \PYG{o}{=} \PYG{n}{Pred}\PYG{p}{(}\PYG{n}{CONF}\PYG{p}{)} \PYG{c+c1}{\PYGZsh{} Fluent predicates}
\PYG{n}{AtPose} \PYG{o}{=} \PYG{n}{Pred}\PYG{p}{(}\PYG{n}{BLOCK}\PYG{p}{,} \PYG{n}{POSE}\PYG{p}{)}
\PYG{n}{HandEmpty} \PYG{o}{=} \PYG{n}{Pred}\PYG{p}{()}
\PYG{n}{Holding} \PYG{o}{=} \PYG{n}{Pred}\PYG{p}{(}\PYG{n}{BLOCK}\PYG{p}{)}
\PYG{n}{Safe} \PYG{o}{=} \PYG{n}{Pred}\PYG{p}{(}\PYG{n}{BLOCK}\PYG{p}{,} \PYG{n}{BLOCK}\PYG{p}{,} \PYG{n}{POSE}\PYG{p}{)} \PYG{c+c1}{\PYGZsh{} Derived predicates}
\PYG{n}{IsKin} \PYG{o}{=} \PYG{n}{Pred}\PYG{p}{(}\PYG{n}{POSE}\PYG{p}{,} \PYG{n}{CONF}\PYG{p}{)} \PYG{c+c1}{\PYGZsh{} Static predicates}
\PYG{n}{IsCollisionFree} \PYG{o}{=} \PYG{n}{Pred}\PYG{p}{(}\PYG{n}{BLOCK}\PYG{p}{,} \PYG{n}{POSE}\PYG{p}{,} \PYG{n}{BLOCK}\PYG{p}{,} \PYG{n}{POSE}\PYG{p}{)}

\PYG{n}{actions} \PYG{o}{=} \PYG{p}{[}
  \PYG{n}{Action}\PYG{p}{(}\PYG{n}{name}\PYG{o}{=}\PYG{l+s+s1}{\PYGZsq{}pick\PYGZsq{}}\PYG{p}{,} \PYG{n}{parameters}\PYG{o}{=}\PYG{p}{[}\PYG{n}{B1}\PYG{p}{,} \PYG{n}{P1}\PYG{p}{,} \PYG{n}{Q1}\PYG{p}{],}
    \PYG{n}{condition}\PYG{o}{=}\PYG{n}{And}\PYG{p}{(}\PYG{n}{AtPose}\PYG{p}{(}\PYG{n}{B1}\PYG{p}{,} \PYG{n}{P1}\PYG{p}{),} \PYG{n}{HandEmpty}\PYG{p}{(),} \PYG{n}{AtConf}\PYG{p}{(}\PYG{n}{Q1}\PYG{p}{),} \PYG{n}{IsKin}\PYG{p}{(}\PYG{n}{P1}\PYG{p}{,} \PYG{n}{Q1}\PYG{p}{)),}
    \PYG{n}{effect}\PYG{o}{=}\PYG{n}{And}\PYG{p}{(}\PYG{n}{Holding}\PYG{p}{(}\PYG{n}{B1}\PYG{p}{),} \PYG{n}{Not}\PYG{p}{(}\PYG{n}{AtPose}\PYG{p}{(}\PYG{n}{B1}\PYG{p}{,} \PYG{n}{P1}\PYG{p}{)),} \PYG{n}{Not}\PYG{p}{(}\PYG{n}{HandEmpty}\PYG{p}{()))),}
  \PYG{n}{Action}\PYG{p}{(}\PYG{n}{name}\PYG{o}{=}\PYG{l+s+s1}{\PYGZsq{}place\PYGZsq{}}\PYG{p}{,} \PYG{n}{parameters}\PYG{o}{=}\PYG{p}{[}\PYG{n}{B1}\PYG{p}{,} \PYG{n}{P1}\PYG{p}{,} \PYG{n}{Q1}\PYG{p}{],}
    \PYG{n}{condition}\PYG{o}{=}\PYG{n}{And}\PYG{p}{(}\PYG{n}{Holding}\PYG{p}{(}\PYG{n}{B1}\PYG{p}{),} \PYG{n}{AtConf}\PYG{p}{(}\PYG{n}{Q1}\PYG{p}{),} \PYG{n}{IsKin}\PYG{p}{(}\PYG{n}{P1}\PYG{p}{,} \PYG{n}{Q1}\PYG{p}{),}
      \PYG{n}{ForAll}\PYG{p}{([}\PYG{n}{B2}\PYG{p}{],} \PYG{n}{Or}\PYG{p}{(}\PYG{n}{Equal}\PYG{p}{(}\PYG{n}{B1}\PYG{p}{,} \PYG{n}{B2}\PYG{p}{),} \PYG{n}{Safe}\PYG{p}{(}\PYG{n}{B2}\PYG{p}{,} \PYG{n}{B1}\PYG{p}{,} \PYG{n}{P1}\PYG{p}{)))),}
    \PYG{n}{effect}\PYG{o}{=}\PYG{n}{And}\PYG{p}{(}\PYG{n}{AtPose}\PYG{p}{(}\PYG{n}{B1}\PYG{p}{,} \PYG{n}{P1}\PYG{p}{),} \PYG{n}{HandEmpty}\PYG{p}{(),} \PYG{n}{Not}\PYG{p}{(}\PYG{n}{Holding}\PYG{p}{(}\PYG{n}{B1}\PYG{p}{)))),}
  \PYG{n}{Action}\PYG{p}{(}\PYG{n}{name}\PYG{o}{=}\PYG{l+s+s1}{\PYGZsq{}move\PYGZsq{}}\PYG{p}{,} \PYG{n}{parameters}\PYG{o}{=}\PYG{p}{[}\PYG{n}{Q1}\PYG{p}{,} \PYG{n}{Q2}\PYG{p}{],}
    \PYG{n}{condition}\PYG{o}{=}\PYG{n}{AtConf}\PYG{p}{(}\PYG{n}{Q1}\PYG{p}{),}
    \PYG{n}{effect}\PYG{o}{=}\PYG{n}{And}\PYG{p}{(}\PYG{n}{AtConf}\PYG{p}{(}\PYG{n}{Q2}\PYG{p}{),} \PYG{n}{Not}\PYG{p}{(}\PYG{n}{AtConf}\PYG{p}{(}\PYG{n}{Q1}\PYG{p}{))))]}
\PYG{n}{axioms} \PYG{o}{=} \PYG{p}{[}
  \PYG{n}{Axiom}\PYG{p}{(}\PYG{n}{effect}\PYG{o}{=}\PYG{n}{Safe}\PYG{p}{(}\PYG{n}{B2}\PYG{p}{,} \PYG{n}{B1}\PYG{p}{,} \PYG{n}{P1}\PYG{p}{),} \PYG{c+c1}{\PYGZsh{} Infers B2 is at a safe pose wrt B1 at P1}
        \PYG{n}{condition}\PYG{o}{=}\PYG{n}{Exists}\PYG{p}{([}\PYG{n}{P2}\PYG{p}{],} \PYG{n}{And}\PYG{p}{(}\PYG{n}{AtPose}\PYG{p}{(}\PYG{n}{B2}\PYG{p}{,} \PYG{n}{P2}\PYG{p}{),} \PYG{n}{IsCollisionFree}\PYG{p}{(}\PYG{n}{B1}\PYG{p}{,} \PYG{n}{P1}\PYG{p}{,} \PYG{n}{B2}\PYG{p}{,} \PYG{n}{P2}\PYG{p}{))))]}

\PYG{n}{cond\PYGZus{}streams} \PYG{o}{=} \PYG{p}{[}
  \PYG{n}{GeneratorStream}\PYG{p}{(}\PYG{n}{inputs}\PYG{o}{=}\PYG{p}{[],} \PYG{n}{outputs}\PYG{o}{=}\PYG{p}{[}\PYG{n}{P1}\PYG{p}{],} \PYG{n}{conditions}\PYG{o}{=}\PYG{p}{[],} \PYG{n}{effects}\PYG{o}{=}\PYG{p}{[],}
                  \PYG{n}{generator}\PYG{o}{=}\PYG{k}{lambda}\PYG{p}{:} \PYG{n+nb}{xrange}\PYG{p}{(}\PYG{n}{num\PYGZus{}poses}\PYG{p}{)),} \PYG{c+c1}{\PYGZsh{} Enumerates all the poses}
  \PYG{n}{GeneratorStream}\PYG{p}{(}\PYG{n}{inputs}\PYG{o}{=}\PYG{p}{[}\PYG{n}{P1}\PYG{p}{],} \PYG{n}{outputs}\PYG{o}{=}\PYG{p}{[}\PYG{n}{Q1}\PYG{p}{],} \PYG{n}{conditions}\PYG{o}{=}\PYG{p}{[],} \PYG{n}{effects}\PYG{o}{=}\PYG{p}{[}\PYG{n}{IsKin}\PYG{p}{(}\PYG{n}{P1}\PYG{p}{,} \PYG{n}{Q1}\PYG{p}{)],}
                  \PYG{n}{generator}\PYG{o}{=}\PYG{k}{lambda} \PYG{n}{p}\PYG{p}{:} \PYG{p}{[}\PYG{n}{p}\PYG{p}{]),} \PYG{c+c1}{\PYGZsh{} Inverse kinematics}
  \PYG{n}{TestStream}\PYG{p}{(}\PYG{n}{inputs}\PYG{o}{=}\PYG{p}{[}\PYG{n}{B1}\PYG{p}{,} \PYG{n}{P1}\PYG{p}{,} \PYG{n}{B2}\PYG{p}{,} \PYG{n}{P2}\PYG{p}{],} \PYG{n}{conditions}\PYG{o}{=}\PYG{p}{[],} \PYG{n}{effects}\PYG{o}{=}\PYG{p}{[}\PYG{n}{IsCollisionFree}\PYG{p}{(}\PYG{n}{B1}\PYG{p}{,} \PYG{n}{P1}\PYG{p}{,} \PYG{n}{B2}\PYG{p}{,} \PYG{n}{P2}\PYG{p}{)],}
             \PYG{n}{test}\PYG{o}{=}\PYG{k}{lambda} \PYG{n}{b1}\PYG{p}{,} \PYG{n}{p1}\PYG{p}{,} \PYG{n}{b2}\PYG{p}{,} \PYG{n}{p2}\PYG{p}{:} \PYG{n}{p1} \PYG{o}{!=} \PYG{n}{p2}\PYG{p}{)]} \PYG{c+c1}{\PYGZsh{} Collision checking}

\PYG{n}{constants} \PYG{o}{=} \PYG{p}{[]}
\PYG{n}{initial\PYGZus{}atoms} \PYG{o}{=} \PYG{p}{[}\PYG{n}{AtConf}\PYG{p}{(}\PYG{n}{initial\PYGZus{}config}\PYG{p}{),} \PYG{n}{HandEmpty}\PYG{p}{()]} \PYG{o}{+} \PYGZbs{}
                \PYG{p}{[}\PYG{n}{AtPose}\PYG{p}{(}\PYG{n}{block}\PYG{p}{,} \PYG{n}{pose}\PYG{p}{)} \PYG{k}{for} \PYG{n}{block}\PYG{p}{,} \PYG{n}{pose} \PYG{o+ow}{in} \PYG{n}{initial\PYGZus{}poses}\PYG{o}{.}\PYG{n}{iteritems}\PYG{p}{()]}
\PYG{n}{goal\PYGZus{}formula} \PYG{o}{=} \PYG{n}{And}\PYG{p}{(}\PYG{n}{AtPose}\PYG{p}{(}\PYG{n}{block}\PYG{p}{,} \PYG{n}{pose}\PYG{p}{)} \PYG{k}{for} \PYG{n}{block}\PYG{p}{,} \PYG{n}{pose} \PYG{o+ow}{in} \PYG{n}{goal\PYGZus{}poses}\PYG{o}{.}\PYG{n}{iteritems}\PYG{p}{())}
\PYG{k}{return} \PYG{n}{STRIPStreamProblem}\PYG{p}{(}\PYG{n}{initial\PYGZus{}atoms}\PYG{p}{,} \PYG{n}{goal\PYGZus{}formula}\PYG{p}{,} \PYG{n}{actions}\PYG{o}{+}\PYG{n}{axioms}\PYG{p}{,} \PYG{n}{cond\PYGZus{}streams}\PYG{p}{,} \PYG{n}{constants}\PYG{p}{)}
\end{Verbatim}

\end{footnotesize}
\caption{\algname{} Python code for the example discrete domain} \label{fig:python}
\end{figure*}

\end{document}